%% file: main.tex
\documentclass[letterpaper,11pt]{article} % Anonymized submission
\usepackage{natbib}

\title{Near-Optimal Cryptographic Hardness of Learning With Homogeneous Halfspaces Under Gaussian Marginals}
\author{Jizhou Huang\\Meta\\\texttt{huang.jizhou@wustl.edu}
\and 
Brendan Juba\\Washington Universtiy in St.\ Louis\\
\texttt{bjuba@wustl.edu}}
\usepackage{fullpage}
\input{project-commands}

\date{}

\newtheorem{theorem}{Theorem}
\newtheorem{definition}[theorem]{Definition}
\newtheorem{lemma}[theorem]{Lemma}
\newtheorem{corollary}[theorem]{Corollary}
\newtheorem{proposition}[theorem]{Proposition}
\newtheorem{fact}[theorem]{Fact}
\newtheorem{problem}[theorem]{Problem}
\newtheorem{assumption}[theorem]{Assumption}

\input{figures}

\begin{document}

    \maketitle
    
    \begin{abstract}%
        We study three problems that involve identifying homogeneous halfspaces under Gaussian distributions: agnostic learning, one-sided reliable learning, and fairness auditing. In each of these problems, we are given labeled examples $(\rvector{x}, \rscalar{y})$ drawn from an unknown distribution on $\real|d|\times\binarydomain$, whose marginal distribution on $\rvector{x}$ is standard Gaussian and on $\rscalar{y}$ is arbitrary. The goal of each problem is to output a homogeneous halfspace that approaches the best-fitting homogeneous halfspace in terms of its corresponding loss measure. We prove near-optimal computational hardness results for these problems under the widely believed hardness assumption of the Learning With Errors (LWE) problem. Prior hardness results for these problems were mostly established for general halfspaces; our findings extend some of these hardness results to homogeneous halfspaces. Remarkably, our lower bound strictly generalizes over prior works and narrows the gap between the upper and lower bounds for agnostically learning homogeneous halfspaces under Gaussian marginals.
    \end{abstract}

 \newpage
 
    \section{Introduction}
    \label{sec:introduction}
        Halfspaces, also known as linear threshold functions, constitute one of the most basic and well-studied concept classes in supervised learning. A halfspace is a function $h: \real|d|\to \{\pm 1\}$ of the form $h(\rvector{x}) = \sgn(\innerprod{\cvector{u}}{\rvector{x}} - t)$ for some unit weight vector $\cvector{u}\in\sphere|d - 1|$ and threshold $t > 0$. When $t = 0$, the halfspace is called \emph{homogeneous}. Geometrically, halfspaces represents a rather simple concept---they partition the input space by a hyperplane. Despite the simplicity, identifying halfspaces in the presence of \emph{adversarial} label noise still suffers strong computational barriers in multiple problems even when the underlying distribution of $\rvector{x}$ is Gaussian distributions. 
        
        In this work, we study the computational complexity of an even simpler variant of the question---learning homogeneous halfspaces under Gaussian $\rvector{x}$-marginals under adversarial noise and related problems, specifically: \emph{positive-reliable learning} and \emph{fairness auditing}.

        \subsection{Agnostically Learning Halfspaces}
        \label{sec:agnostically-learning-halfspaces}
            In machine learning, one of the most-studied problems is learning halfspaces. Informally, the learner is given a set of labeled examples, $S = \lbr{(\rvector{x}(1), \rscalar{y}(1)), \ldots, (\rvector{x}(m), \rscalar{y}(m))}$, sampled i.i.d.\ from an unknown distribution $\distr$ on $\real|d|\times\binarydomain$, and the goal is to find a halfspace that correctly labels as many examples from $\distr$ as possible. Many algorithms, such as Perceptron~\citep{rosenblatt1958perceptron} and Support Vector Machines~\citep{vapnik1998statistical}, were developed for this problem. However, understanding the computational complexity of learning halfspaces remains a central challenge in learning theory.
            
            In the classical PAC-learning model \citep{valiant1984theory}, the sample space of the underlying distribution $\distr$ of $(\rvector{x}, \rscalar{y})$ is assumed to be linearly separable, and hence, a halfspace that correctly labels all examples from $\distr$ exists. It turns out that approximately finding such a halfspace is computationally easy~\citep{blumer1989learnability}. However, real-world data is often too complex to be captured by such a simple concept. A more realistic and interesting problem setting is \emph{agnostic} learning \citep{haussler1992decision,kearns1994toward}, where no assumption is made on the underlying distribution $\distr$, and the goal is to find a halfspace that minimizes the \emph{classification error}. Unfortunately, even \emph{weak} learning of this seemingly simple concept class becomes intractable under the agnostic setting (in the absence of assumptions on the $\rvector{x}$-marginal of $\distr$)~\citep{shalev2011learning}. Due to this computational barrier, significant effort has been invested in learning halfspaces in \emph{distribution-specific} settings, which may lend additional leverage to enable efficient algorithms that are robust to adversarial label noise. In this work, we focus on the computational challenge of agnostically learning homogeneous halfspaces under Gaussian $\rvector{x}$-marginals.

            \begin{problem}[Agnostic Learning]
            \label{prob:agnostic-learning}
                 Let $\distr$ of $(\rvector{x}, \rscalar{y})$ be any distribution supported on $\real|d|\times\binarydomain$ with \textbf{standard normal} $\rvector{x}$-marginals, and $\opt = \inf_{\cvector{u}\in\sphere|d-1|}\prob<\distr>{\rscalar{y} \neq \funcsbr{\sgn}[\innerprod{\cvector{u}}{\rvector{x}}]}$. For parameters $\alpha \geq 1$, $\epsilon,\delta\in[0, 1]$, the $(\alpha, \epsilon)$-approximate learning task for homogeneous halfspaces is, given $m$ i.i.d.\ examples from $\distr$, output a $\cvector{v}\in\sphere|d-1|$ such that $\prob<\distr>{\rscalar{y} \neq \funcsbr{\sgn}[\innerprod{\cvector{v}}{\rvector{x}}]}\leq \alpha\opt + \epsilon$ with probability $1 - \delta$. We call it a $\alpha$-approximation if $\epsilon = 0$.
            \end{problem}

            \noindent
            \textbf{Prior Work on \cref{prob:agnostic-learning}: }On the positive side, many efficient algorithms have been proposed for agnostically learning homogeneous halfspaces under various distribution assumptions~\citep{kalai2008agnostically,klivans2009learning,daniely2015ptas,awasthi2017power,diakonikolas2020non,diakonikolas2021agnostic,frei2021agnostic}. Within $\poly[d]$ running time, the best known multiplicative upper bound for homogeneous halfspaces is $\bigO{\opt}$~\citep{diakonikolas2020non} under \emph{log-concave} $\rvector{x}$-marginals, which is also extended to general halfspaces under standard normal $\rvector{x}$-marginals by~\citet{diakonikolas2018learning,diakonikolas2022learning}. However, for additive upper bound, the best known algorithms can only provide $(1, \bigO{1})$-approximation guarantee in polynomial running time, or $(1, \bigO{1/\sqrt{\log d})}$-approximation in quasi-polynomial running time even under standard normal $\rvector{x}$-marginals~\citep{diakonikolas2010bounded,daniely2015ptas}. On the negative side, \citet{diakonikolas2020non} showed that no $\funcsbr{o}[\sqrt{\funcsbr{\log}[1/\opt]}]$-approximation to the optimal homogeneous halfspace can be obtained by minimizing a \emph{convex non-decreasing} surrogate loss, even with standard normal $\rvector{x}$-marginals. More recently, \citet{diakonikolas2023near} found that there exists no polynomial-time $(1, 1/\cscalar{\log}|1/2 + c|d)$-approximation algorithm for \emph{general} (\underline{not} homogeneous) halfspaces for any $c > 0$, even under Gaussian $\rvector{x}$-marginals, given the hardness of a classical cryptographic problem. 

            Based on these findings, it is natural to ask if we can close the gap and obtain tight bounds for \cref{prob:agnostic-learning} in certain settings. We accomplish this by proving that no polynomial-time $(1, 1/\cscalar{\log}|1/2 + c|d)$-approximation algorithm solves \cref{prob:agnostic-learning} (c.f. \cref{thm:hardness-of-learning-homogeneous-halfspace-under-gaussian}), assuming the widely believed sub-exponential hardness of the \emph{Learning With Errors} (LWE) problem (c.f. \definitionref{def:learning-with-errors}), admitting a strict generalization of the cryptographic lower bound in \citet{diakonikolas2023near}.

        \subsection{One-sided Reliable Learning}
        \label{sec:one-sided-reliable-learning}
            Beyond classical agnostic learning, several other learning paradigms involving halfspaces have emerged from practical applications. The problem of \emph{positive-reliable learning} \citep{kalai2012reliable}, originally known as \emph{heuristic} learning \citep{pitt1988computational} (and equivalent to ``abductive'' learning \citep{juba2016learning,durgin2019hardness}), seeks to identify a subset of the population where every individual belongs to the positive class.\footnote{\cite{durgin2019hardness} establish the equivalence of these various formulations.} \cite{kanade2014distribution} and Juba observed that subpopulations defined by $k$-DNFs can be efficiently learned in realizable cases without any distributional assumptions (\cite{bshouty2005maximizing}, meanwhile, attribute this to \cite{valiant1984theory}). When the subset membership is defined by halfspaces, one-sided learning resembles halfspaces learning in the sense that learning halfspaces aims to optimize the objective loss on both sides of the defining hyperplane, while this problem only optimize one-sided loss (c.f. \cref{prob:positive-reliable-learning}). Unfortunately, \citet{bshouty2005maximizing} showed that, in the agnostic setting, if the collection of sub-populations is represented by some concept class with the ability to express \emph{conjunctions} (a subclass of halfspaces), such a sub-population cannot be efficiently learned without distributional assumptions. Therefore, we focus on one-sided learning of halfspace subsets with distributional assumption. In fact, we consider almost the simplest variant of the problem, which is learning the positive side of homogeneous halfspaces under Gaussian $\rvector{x}$-marginals.

            \begin{problem}[Positive-Reliable Learning]
            \label{prob:positive-reliable-learning}
                 Let $\distr$ of $(\rvector{x}, \rscalar{y})$ be any distribution on $\real|d|\times\binarydomain$ with \textbf{standard normal} $\rvector{x}$-marginals, and $1 - \opt = \sup_{\cvector{u}\in\sphere|d-1|}\prob<\distr>{\rscalar{y} = 1\cond \innerprod{\cvector{u}}{\rvector{x}} > 0}$. For parameters $\alpha \geq 1$ and $\epsilon,\delta\in[0, 1]$, the $(\alpha, \epsilon)$-approximate positive-reliable learning task for homogeneous halfspaces is, given $m$ i.i.d.\ examples from $\distr$, to output a $\cvector{v}\in\sphere|d-1|$ such that $\prob<\distr>{\rscalar{y} = 1\cond\innerprod{\cvector{v}}{\rvector{x}} > 0}\geq 1 - \alpha\opt - \epsilon$ with probability $1 - \delta$. We call it a $\alpha$-approximation if $\epsilon = 0$.
            \end{problem}

            \noindent
            \textbf{Prior Work on \cref{prob:positive-reliable-learning}: }Beyond the aforementioned works that noted elimination learns $k$-DNFs, \citet{huang2025distributionspecific,huang2025personalized} presented approximation algorithms for positive-reliable learning of homogeneous halfspaces. In particular, \citet{huang2025distributionspecific} gave a $\bigO*{\opt|-1/2|}$-approximation algorithm under Gaussian marginals, and \citet{huang2025personalized} gave a $\bigO{\opt|-3/4|}$-approximation algorithm under a more general \emph{well-behaved} marginal distributions. Meanwhile, \citet{huang2025distributionspecific} proved that there does not exist any polynomial-time $(1, 1/\cscalar{\log}|1/2 + c|d)$-approximation algorithm for \cref{prob:positive-reliable-learning} on general halfspaces with the additional requirement that $1/2 \leq\mu_1\leq\prob{\innerprod{\cvector{u}}{\rvector{x}} > t}\leq \mu_2, \forall\cvector{u}\in\sphere|d-1|$, for any $c > 0$, assuming the intractability of the LWE problem. 
            
            Obviously, the negative result of \citet{huang2025distributionspecific} does not meet our objectives: it is for the more general class of halfspaces with thresholds and only applies to such halfspaces subject to the population size constraints, $\mu_1\leq\prob<\distr<\rvector{x}>>{\innerprod{\cvector{u}}{\rvector{x}} > t}\leq \mu_2, \forall\cvector{u}\in\sphere|d-1|$. In this work, we improve their result by showing non-existence of polynomial-time $(1, 1/\cscalar{\log}|1/2 + c|d)$-approximation algorithms for \cref{prob:positive-reliable-learning} (c.f. \cref{thm:hardness-of-positive-reliable-learning-over-homogeneous-halfspace}), and indeed obtain the same cryptographic hardness over general halfspaces without the population constraints.

        \subsection{Subgroup Fairness Auditing}
        \label{sec:subgroup-fairness-auditing}
            Another important halfspace-related problem arises in the context of \emph{algorithmic fairness auditing}. The goal in the problem is to verify if any subgroup of a given collection of subgroups is discriminated against by the underlying decision rule in terms of certain fairness criteria. One of the most general definition of fairness is Statistical Parity Subgroup Fairness (SPSF), which was introduced by \citet{kearns2018preventing}. While \citet{kearns2018preventing} was considering a broader class of subgroups, including conjunctions and halfspaces, we focus on the class of subgroups defined on homogeneous halfspaces. 
            \begin{definition}[SPSF For Homogeneous Halfspace]
            \label{def:statistical-parity-subgroup-fairness}
                Let $\distr$ of $(\rvector{x}, \rscalar{y})$ be any distribution supported on $\real|d|\times\binarydomain$, and $S = \lbr{\rvector{x}\in\real|d|\cond \innerprod{\cvector{u}}{\rvector{x}} > 0}$ be a subgroup defined on a homogeneous halfspace with normal vector $\cvector{u}\in\sphere|d-1|$. We define the unfairness of the subgroup $S$ w.r.t.\ $\distr$ as
                \begin{equation*}
                    \funcsbr{u}<\distr>[\cvector{u}] = \abs{\prob<\rvector{x}\sim\distr<\rvector{x}>>{\innerprod{\cvector{u}}{\rvector{x}} > 0}\prob<\rscalar{y}\sim\distr<\rscalar{y}>>{\rscalar{y} = 1} - \prob<(\rvector{x}, \rscalar{y})\sim\distr>{\rscalar{y} = 1\cap \innerprod{\cvector{u}}{\rvector{x}} > 0}}.
                \end{equation*}
            \end{definition}
            
            \begin{problem}[SPSF Auditing for Homogeneous Halfspaces]
            \label{prob:fairness-auditing}
                 Let $\distr$ of $(\rvector{x}, \rscalar{y})$ be any distribution supported on $\real|d|\times\binarydomain$ with \textbf{standard normal} $\rvector{x}$-marginals, and $\opt = \sup_{\cvector{u}\in\sphere|d-1|}\funcsbr{u}<\distr>[\cvector{u}]$. For parameters $\alpha, \epsilon,\delta\in[0, 1]$, the $(\alpha, \epsilon)$-approximate auditing task for homogeneous halfspace subgroups is, given $m$ i.i.d.\ examples from $\distr$, output a $\cvector{v}\in\sphere|d-1|$ such that $\funcsbr{u}<\distr>[\cvector{v}]\geq \alpha\opt - \epsilon$ with probability $1 - \delta$. We call it a $\alpha$-approximation if $\epsilon = 0$.
            \end{problem}
            \noindent
            \textbf{Prior Work on \cref{prob:fairness-auditing}: }Remarkably, \citet{kearns2018preventing} established the sub-exponential hardness of SPSF auditing on any simple concept classes in the distribution-free setting. They gave an efficient reduction from weak agnostic learning of such a concept classes to SPSF auditing on the same class of subgroups. Despite the hardness, auditing remains an important problem in computational fairness learning, as multiple works have observed that learning fair classifiers is computationally equivalent to auditing for the given subgroups~\citep{hebert2018multicalibration,kearns2018preventing,kim2018fairness}. Contrary to the stagnation of algorithmic auditing over subgroup classes of potentially exponential size, much progress has been made toward fairness learning/auditing for small subgroup classes that are affordable to enumerate~\citep{dwork2012fairness,kusner2017,agarwal2018reductions,bechavod2020metric,bechavod2023individually,wang2023scalable}. Recently, \citet{hsu2024distribution} proved a surprisingly strong computational barrier of SPSF auditing for general halfspaces. They proved that neither $1/\poly[d]$-approximation nor $(1, 1/\cscalar{\log}|1/2 + c|d)$-approximation algorithms running in polynomial time exist for \cref{prob:fairness-auditing} on general halfspaces subgroups for any $c > 0$, assuming the hardness of the LWE problem.

            In this paper, we extend the cryptographic hardness results of \citet{hsu2024distribution} from subgroups defined on general halfspaces to homogeneous halfspaces (c.f. \cref{thm:hardness-of-auditing-halfspace-under-gaussian-in-additive-form,thm:hardness-of-auditing-halfspace-under-gaussian-in-multiplicative-form}).

        \subsection{Our Results}
        \label{sec:our-results}
            We first introduce an informal definition of the LWE problem. In the LWE problem, we are given $m$ examples, $(\rvector{x}(1), \rscalar{y}(1)), \ldots, (\rvector{x}(m), \rscalar{y}(m))$, the goal is to distinguish the following two cases:
            \begin{itemize}
                \item Each $\rvector{x}(i)$ is drawn u.a.r. from $\integer<q>|d|$, and each $\rscalar{y}(i)$ is generated as $\rscalar{y}(i) = \innerprod{\cvector{s}}{\rvector{x}(i)} + \rscalar{z}(i)$ by some hidden secret vector $\cvector{s}\in\integer<q>|d|$ and some independent discrete Gaussian noise $\rscalar{z}(i)\in\integer<q>|d|$.
                \item Each $\rvector{x}(i)$ is drawn u.a.r. from $\integer<q>|d|$, and each $\rscalar{y}(i)$ is sampled independently and u.a.r. from $\integer<q>$.
            \end{itemize}
            
            Under the standard sub-exponential hardness assumption of LWE, we establish strong inapproximability for three problems: agnostically learning homogeneous halfspaces, positive-reliable learning of homogeneous halfspaces, and fairness auditing for homogeneous halfspace subgroups. We give their informal statement here and will formally state them in \cref{sec:hardness-of-identifying-special-halfspaces-under-gaussian}.

            Notice that \cref{thm:hardness-of-learning-homogeneous-halfspace-under-gaussian} (resp. \cref{thm:hardness-of-positive-reliable-learning-over-homogeneous-halfspace,thm:hardness-of-auditing-halfspace-under-gaussian-in-additive-form}) is an immediate implication of \corollaryref{cor:hardness-of-learning-homogeneous-halfspace-under-gaussian} (resp. \cref{cor:hardness-of-positive-reliable-learning-over-homogeneous-halfspace,cor:hardness-of-auditing-halfspace-under-gaussian-in-additive-form}), where we just replace $\epsilon$ by its upper bound $\log^\gamma d$. Since the upper bound holds for any $\gamma > 1/2$, we have that $\log^\gamma d = \cscalar{\log}|1/2 + c|d$ for any $c > 0$. 

            \begin{theorem}[Hardness of Agnostically Learning]
            \label{thm:hardness-of-learning-homogeneous-halfspace-under-gaussian}
                If LWE cannot be solved in $2^{\cscalar{d}|1 - \bigM{1}|}$ time, then, for any $c > 0$, no $\poly[d]$-time $(1, 1/\cscalar{\log}|1/2 + c|d)$-approximation algorithm for \cref{prob:agnostic-learning} exists.
            \end{theorem}
            
            We would like to highlight that our result in \cref{thm:hardness-of-learning-homogeneous-halfspace-under-gaussian} is a strict generalization of that in \citet{diakonikolas2023near} because their negative result does not necessarily hold for homogeneous halfspaces. While not tightly matching the best known upper bound, our lower bound $\opt + 1/\cscalar{\log}|1/2 + c|d$ is a nontrivial step forward on closing the gap in this research line.

            \begin{theorem}[Hardness of Positive-Reliable Learning]
            \label{thm:hardness-of-positive-reliable-learning-over-homogeneous-halfspace}
                If LWE cannot be solved in $2^{\cscalar{d}|1 - \bigM{1}|}$ time, then, for any $c > 0$, no $\poly[d]$-time $(1, 1/\cscalar{\log}|1/2 + c|d)$-approximation algorithm for \cref{prob:positive-reliable-learning} exists.
            \end{theorem}

\noindent
            Although \citet{huang2025distributionspecific} also obtained a $1 - \opt - 1/\cscalar{\log}|1/2 + c|d$ cryptographic lower bound for positive-reliable learning with Gaussian $\rvector{x}$-marginals, we emphasize that our hardness result is a strict generalization of theirs. We note that their result only works for intersections of pairs of parallel halfspaces, which forms a class of ``band-like'' representations. Such a concept class is uncommon and is more expressive than general halfspaces, which makes their hardness result less general than \cref{thm:hardness-of-positive-reliable-learning-over-homogeneous-halfspace}.

            \begin{theorem}[Hardness of Fairness Auditing--Additive]
            \label{thm:hardness-of-auditing-halfspace-under-gaussian-in-additive-form}
                If LWE cannot be solved in $2^{\cscalar{d}|1 - \bigM{1}|}$ time, then, for any $c > 0$, no $\poly[d]$-time $(1, 1/\cscalar{\log}|1/2 + c|d)$-approximation algorithm for \cref{prob:fairness-auditing} exists.
            \end{theorem}
            
            \begin{theorem}[Hardness of Fairness Auditing--Multiplicative]\label{thm:hardness-of-auditing-halfspace-under-gaussian-in-multiplicative-form}
                If LWE cannot be solved in $2^{\cscalar{d}|1 - \bigM{1}|}$ time, then no $\poly[d]$-time $1/\poly[d]$-approximation algorithm for \cref{prob:fairness-auditing} exists.
            \end{theorem}

\noindent
            Similar to the previous remarks, \cref{thm:hardness-of-auditing-halfspace-under-gaussian-in-additive-form,thm:hardness-of-auditing-halfspace-under-gaussian-in-multiplicative-form} are strict generalizations of those in \citet{hsu2024distribution}, as their results does not necessarily hold for homogeneous halfspaces. 
        \subsection{Our Techniques}
        \label{sec:our-techniques}

            Our hardness results are based on a reduction from a continuous variant of LWE (cLWE) problem to a distinguishability variant (c.f. \cref{thm:indistinguishability-of-agnostic-learning,thm:indistinguishability-of-one-sided-homogeneous-halfspaces,thm:indistinguishability-of-fairness-auditing}) of each of the problems we considered. With appropriate choice of parameters, \citet{diakonikolas2023near} and \citet{gupte2022continuous} showed that cLWE is as hard as the generic LWE problem (c.f. \definitionref{def:learning-with-errors}). Our reduction method is based on that of \citet{diakonikolas2023near} and \citet{hsu2024distribution}, where they take the distribution (either alternative or null) from an LWE instance and map its labels from $[0, q)$ to $\binarydomain$ using a simple rule. For each of our three problems, we want the image of the distribution to inherit the information bias from the original distribution (alternative or null) in some form such that the image of the alternative hypothesis distribution is distinguishable from the image of the null hypothesis distribution using solutions to the problem. Obviously, we can achieve this by showing no halfspace is information-theoretically different from random guessing if the original distribution is from the null hypothesis, while a halfspace of significant advantage over random guessing exists otherwise. Although we share a similar goal with \citet{diakonikolas2023near,hsu2024distribution}, our proof differs drastically. (Meanwhile, the hardness result of \citet{huang2025distributionspecific} is based on a reduction from agnostic halfspace learning to reliable learning of halfspaces, and thus follows from the hardness of agnostic learning.)

            Both \citet{diakonikolas2023near} and \citet{hsu2024distribution} use the secret vector from the LWE instance to construct two parallel halfspaces with small agreement. They showed that the sum of the classification loss of the two hypothesis must significantly different from $1$, which implies one of them must have significant advantage over random guessing. The key step in both works is to reduce the dependency of the loss sum from the entirety of $\real|d|$ to a small band-like area located between the defining hyperplanes of two halfspaces. The advantage of the halfspace over random guessing is approximately equal to the probability mass of this small area. Observe that such a construction requires at least one of the halfspaces to have non-zero threshold, which prevents applying the argument to homogeneous halfspaces.

            Our proof is conceptually more straightforward: we directly argue that the homogeneous halfspace defined by the secret vector of the LWE instance has significant advantage over random guessing in terms of the corresponding loss measure. We observe that the objective loss of each of the three problems can be rewritten into a summation of a random guessing loss (e.g. $1/2$ for learning halfspaces) and an advantage term. Importantly, the advantage term can be represented by a \emph{partial} convolution between a square wave of period $q$ and the Gaussian density. By replacing the square wave by its \emph{Fourier series} and computing the partial convolution on each Fourier basis element, we find that the advantage term is dominated by the convolution of the first basis element, which eventually yields our advantage lower bound by a Gaussian tail analysis.

\iffalse{
%% We can omit this at the moment, unless someone asks for it. I suspect this discussion will just be more confusing than helpful to our reviewer.
        \subsection{Related Work}
            The problem, \emph{learning reference classes} \citep{pmlr-v89-hainline19a}, resembles our \cref{prob:positive-reliable-learning} in many ways: it only additionally requires all the candidate hypotheses to include a given point $\rvector{x}|*|\in\real|d|$. The problem was initially considered in classical PAC-learning setting by \citet{juba2016learning} as \emph{Precondition Search}, and they proved its PAC learnability over subsets defined by $k$-DNFs without any distributional assumptions. In the agnostic setting, \citet{juba2020more} showed that any representation class of the subsets with ability to express \emph{conjunctions} cannot be efficiently learned as a reference class without distributional assumptions. A majority of the work being made for learning reference classes is algorithmic results, and most of them were considering subpopulations defined by $k$-DNFs~\citep{juba2016learning,juba2016conditional,pmlr-v89-hainline19a,juba2020more,liang2022conditional}. More recently, \citet{huang2025personalized} presented $\bigO{\opt|-3/4|}$-approximation algorithms for learning homogeneous halfspace reference classes under the well-behaved $\rvector{x}$-marginals.
         }\fi
            
    \section{Preliminaries}
    \label{sec:preliminary}
        \textbf{Notations:} We use lowercase italic font characters to represent scalars, e.g.\ $\cscalar{x}\in\real$, lowercase bold italic font characters to represent vectors, e.g.\ $\cvector{x}\in\real|d|$, and uppercase bold italic font characters to represent matrices, e.g.\ $\cmatrix{A}\in\real|m\times d|$. In particular, subscripts will be used to index the coordinates of any vector, e.g., $\cscalar{x}<i>$ represents the $i$th coordinate of the vector $\cvector{x}$. However, for random variables, we use lowercase normal font characters to represent random scalars, e.g.\ $\rscalar{x}\in\real$, and lowercase bold normal font characters to represent random vectors, e.g.\ $\rvector{x}\in\real|d|$ and $\rmatrix{A}\in \real|m\times d|$.

        For probabilistic notation, we use $\distr<\rvector{x}>$ to denote the marginal distribution of $\distr$ on the domain of $\rvector{x}$, $\prob<\rvector{x}\sim\distr>{\rvector{x}\in \cscalar{E}}$ to denote the probability of an event $\cscalar{E}$, and $\expect<\rvector{x}\sim\distr>{\funcsbr{f}[\rvector{x}]}$ to denote the expectation of some statistic $\funcsbr{f}[\rvector{x}]$. 
        % In particular, for an i.i.d.\ sample $\distr*\sim\distr$, we define the empirical probability and expectation as
        % \begin{align*}
        %     \prob<\rvector{x}\sim\distr*>{\rvector{x}\in\cscalar{E}} =& \frac{1}{\abs*{\distr*}}\sum_{\rvector{x}\in\distr*}\indicator[\rvector{x}\in\cscalar{E}]\\
        %     \expect<\rvector{x}\sim\distr*>{\funcsbr{f}[\rvector{x}]} =& \frac{1}{\abs*{\distr*}}\sum_{\rvector{x}\in\distr*}\funcsbr{f}[\rvector{x}].
        % \end{align*}
        For simplicity of notation, we may drop $\distr$ from the subscript when it is clear from the context, i.e., we may simply write $\prob{E}, \expect{f}$ for $\prob<\rvector{x}\sim\distr>{E}, \expect<\rvector{x}\sim\distr>{f}$. For Gaussian distributions, we write $\gaussian|d|[\mu][\sigma^2]$ for $d$-dimensional isotropic Gaussian distributions. We use $\funcsbr{\phi}<\mu, \sigma>$ to denote the density of $\gaussian[\mu][\sigma^2]$, and simply write $\phi$ for $\gaussian$.

        Define $\subsets* = \lbr*{\rvector{x}\cond \rvector{x}\notin\subsets}$ for the set complement. We also denote $\sphere|d-1|:=\lbr{\rvector{x}\in \real|d| \cond \norm{\rvector{x}}<2>= 1}$, $\real<q>:=[0, q)$, $\integer<q>:=\lbr{0, 1, \ldots, q-1}$, and $\modulo<q>:\real|d|\rightarrow \real<q>$ for the unique translation of the input by $q\integer|d|$ to $\real<q>$ for $q\in\naturals$.

        \noindent\textbf{Learning With Errors: }The LWE problem in \cref{sec:our-results} was introduced by \citet{regev2009lattices}. We formally define the problem of LWE, following \citet{diakonikolas2023near}:
        \begin{definition}[Learning With Errors]
        \label{def:learning-with-errors}
            For $m,d\in\naturals$, $q\in\real<+>$, let $\distr<sample>,\distr<secret>,\distr<noise>$ be distributions on $\real|d|, \real|d|, \real$ respectively. In the LWE$(m, \distr<sample>,\distr<secret>,\distr<noise>, \mathrm{mod}_q)$ problem, with $m$ independent samples $\lbr*{(\rvector{x}(1),\rscalar{y}(1)), \ldots, (\rvector{x}(m), \rscalar{y}(m))}$, we want to distinguish between the following two cases:
            \begin{itemize}
                \item[\textbf{Alternative hypothesis}:]each $(\rvector{x}(i), \rscalar{y}(i))$ is generated as $\rscalar{y}(i) = \mathrm{mod}_q(\innerprod*{\rvector{x}(i)}{\cvector{s}} + \rscalar{z}(i))$, where $\rvector{x}(i)\sim\distr<sample>, \cvector{s}\sim\distr<secret>, \rscalar{z}(i)\sim\distr<noise>$.
                \item[\textbf{Null hypothesis}:]each $\rscalar{y}(i)$ is sampled uniformly at random on the support of its marginal distribution in the alternative hypothesis, independent of $\rvector{x}(i)\sim\distr<sample>$.
            \end{itemize}
        \end{definition}
        
        An algorithm is said to be able to \emph{solve the LWE problem with $\Delta$ advantage} if the probability that the algorithm outputs ``alternative hypothesis'' is $\Delta$ larger than the probability that it outputs ``null hypothesis'' when the given data is sampled from the alternative hypothesis distribution. 
        % LWE is widely believed to be computationally hard, formalized as follows.
    
        \begin{assumption}[Sub-exponential LWE Assumption \citep{lindner2011better}]
        \label{asp:sub-exponential-assumption-of-lwe}
            Let $C > 0$ be a sufficiently large constant and $q\in\naturals$. For any $\kappa\in\naturals, \epsilon\in(0, 1)$, the problem $LWE(\cscalar{2}|\bigO{\cscalar{d}|\epsilon|}| , \integer<q>|d|, \integer<q>|d| ,$ $\gaussian[0][\cscalar{\sigma}|2|], \modulo<q>)$ with $q\leq \cscalar{d}|\kappa|$, $\sigma = C\sqrt{d}$ cannot be solved in $\cscalar{2}|\bigO{\cscalar{d}|\epsilon|}|$ time with $\cscalar{2}|\bigO{-\cscalar{d}|\epsilon|}|$ advantage.
        \end{assumption}

        \noindent
        This conjecture is widely used and consistent with the current state of the art. \citet{regev2009lattices} established the \emph{quantum} hardness of LWE by a quantum reduction from approximating the decision version of the Shortest Vector Problem (GapSVP) to LWE under similar $d,q,\sigma$ parameters. Meanwhile, \citet{peikert2009public} gave a classical version of the reduction, which extends the hardness of LWE to the classical computational model. The fastest known algorithm for solving GapSVP takes $\cscalar{2}|\bigO{d}|$ time. Therefore, disproving this conjecture would represent a significant breakthrough. 
        Our hardness results are obtained via a reduction from a \emph{continuous} variant of LWE (cLWE), which was connected to the original LWE problem by \citet{gupte2022continuous} and \citet{diakonikolas2023near}
        % , such that the secret vector is from the unit sphere $\sphere|d-1|$ and the sample is from a standard normal distribution
        :

        \begin{lemma}[Proposition 2.4 of \cite{diakonikolas2023near}]\label{lma:sub-exponential-hardness-of-clwe}
            Under \cref{asp:sub-exponential-assumption-of-lwe}, for any $d\in\naturals$, any constants $\kappa\in\naturals, \epsilon\in(0, 1), \gamma\in\real<+>$ and any $\log^\gamma d\leq \cscalar{\eta}\leq Cd$ where $C>0$ is a sufficiently small universal constant, the problem LWE$(\cscalar{d}|\bigO{\cscalar{\eta}|\epsilon|}|, \gaussian|d|, \sphere|d-1|, \gaussian[0][\cscalar{\sigma}|2|], \modulo<T>)$ over $\real|d|$ with $\sigma\geq \cscalar{\eta}|-\kappa|$ and $T = 1/C'\sqrt{\cscalar{\eta}\log d}$, where $C'>0$ is a sufficiently large universal constant, cannot be solved in time $\cscalar{d}|\bigO{\cscalar{\eta}|\epsilon|}|$ with $\cscalar{d}|-\bigO{\cscalar{\eta}|\epsilon|}|$ advantage.
        \end{lemma}
        
        % The distribution-specific lowers we will prove later are built upon the above hardness assumption of the cLWE problem.
        
    \section{Existence Of Non-trivial Halfspace In Alternative Hypothesis}
    \label{sec:existence-of-non-trivial-halfspace-in-alternative-hypothesis}
        The goal of this section is to establish a key observation that helps us prove the main hardness results in Section~\ref{sec:hardness-of-identifying-special-halfspaces-under-gaussian}. Our strategy for showing cryptographic lower bounds is to reduce the continuous LWE problem (cLWE) to each of the three problems introduced in \cref{sec:agnostically-learning-halfspaces,sec:one-sided-reliable-learning,sec:subgroup-fairness-auditing}. The reduction simply transforms samples of continuous labels from a cLWE instance into samples of $\binarydomain$ labels as described in \propositionref{prop:lower-bound-of-the-probability-mass-of-alternating-bands-in-halfspace}. We wish to argue that, after applying the reduction, the resulting \emph{alternative hypothesis} would present certain properties that the \emph{null hypothesis} does not possess. Note that an algorithm that efficiently learns such properties is then able to distinguish the two hypotheses after the reduction. Consequently, such an algorithm could solve the cLWE problem, which contradicts the sub-exponential hardness assumption of cLWE (c.f. \lemmaref{lma:sub-exponential-hardness-of-clwe}). 

        The main challenge lies in proving the properties are satisfied in the alternative hypothesis. In particular, we show that the homogeneous halfspace, whose normal vector is the secret vector $\cvector{s}\in\real|d|$ from the alternative hypothesis of a cLWE instance, will exhibit some quantifiable advantages on the reduced samples over random guessing w.r.t.\ the joint probability $\prob{\rscalar{y} = +1\cap\innerprod{\cvector{u}}{\rvector{x}}\geq t}$. 

        \begin{proposition}[Gaussian Tail Lower Bound on Noisy Alternating Bands]
        \label{prop:lower-bound-of-the-probability-mass-of-alternating-bands-in-halfspace}
            For some $\sigma, T > 0$ such that $T\leq 1/\pi$, and $\cvector{u}\in\sphere|d-1|$, let $\distr '$ of $(\rvector{x}, \rscalar{y'})$ be a distribution supported on $\real|d|\times [0, T)$ such that $\distr<\rvector{x}>' = \gaussian|d|$ and $\rscalar{y'} = \modulo<T>[\innerprod{\cvector{u}}{\rvector{x}} + \rscalar{z}]$ for $\rscalar{z}\sim\gaussian[0][\sigma^2]$ being independently sampled from $\rvector{x}$. Define $\distr$ of $(\rvector{x}, \rscalar{y})$ to be the distribution derived from $\distr'$, where $\distr<\rvector{x}> = \distr<\rvector{x}>'$, $\rscalar{y}=1$ if $\rscalar{y'}\leq T/2$, and $\rscalar{y}=-1$ otherwise.
            % \begin{equation*}
            %     \rscalar{y}=
            %     \begin{cases}
            %         +1,\quad \text{if }\rscalar{y'}\leq T/2
            %         \\
            %         -1,\quad \text{otherwise}.
            %     \end{cases}
            % \end{equation*}
            Then, for any $k\in\lbr{0}\cup\integer<+>$, it holds that
            \begin{equation*}
                \prob<(\rvector{x}, \rscalar{y})\sim\distr>{\rscalar{y} = +1\cap\innerprod{\cvector{u}}{\rvector{x}}\geq kT} \geq \frac{1}{2}\prob<\rvector{x}\sim\distr<\rvector{x}>>{\innerprod{\cvector{u}}{\rvector{x}}\geq kT} + \frac{T\cscalar{e}|-k^2T^2 - 2\pi^2\sigma^2/T^2|}{4\pi^2}.
            \end{equation*}
        \end{proposition}
        \begin{proof}
            For conciseness of the proof, we define $\subsets<T> = \union<i\in\integer>[iT, iT + T/2)$. Then, by our definition of $\rscalar{y'}$ and $\rscalar{y}$, we equivalently have that $\rscalar{y}=1$ for $\innerprod{\cvector{u}}{\rvector{x}} + \rscalar{z}\in \subsets<T>$, and $\rscalar{y}=-1$ for $\innerprod{\cvector{u}}{\rvector{x}} + \rscalar{z}\in \subsets*<T>$.
            % \begin{equation*}
            %     \rscalar{y}=
            %     \begin{cases}
            %         +1,\quad \text{if }\innerprod{\cvector{u}}{\rvector{x}} + \rscalar{z}\in \subsets<T>
            %         \\
            %         -1,\quad \text{if }\innerprod{\cvector{u}}{\rvector{x}} + \rscalar{z}\in \subsets*<T>.
            %     \end{cases}
            % \end{equation*}
            Denote $\distr<\rvector{x}, \rscalar{z}>$ to be the product distribution of $\distr<\rvector{x}>$ and $\gaussian[0][\sigma^2]$. It holds that
            \begin{align*}
                &\prob<(\rvector{x}, \rscalar{y})\sim\distr>{\rscalar{y} = +1\cap\innerprod{\cvector{u}}{\rvector{x}}\geq kT}
                \\
                \ceq[i]&\prob<(\rvector{x}, \rscalar{z})\sim\distr<\rvector{x}, \rscalar{z}>>{\innerprod{\cvector{u}}{\rvector{x}} + \rscalar{z}\in \subsets<T>\cap\innerprod{\cvector{u}}{\rvector{x}}\geq kT} 
                \\
                \ceq[ii]& \frac{1}{2}\prob<\distr<\rvector{x}>>{\innerprod{\cvector{u}}{\rvector{x}}\geq kT} + \frac{1}{2}\sbr{\prob<\distr<\rvector{x}, \rscalar{z}>>{\innerprod{\cvector{u}}{\rvector{x}}+ \rscalar{z}\in \subsets<T>\cap \innerprod{\cvector{u}}{\rvector{x}}\geq kT} - \prob<\distr<\rvector{x}, \rscalar{z}>>{\innerprod{\cvector{u}}{\rvector{x}}+ \rscalar{z}\in \subsets*<T>\cap \innerprod{\cvector{u}}{\rvector{x}}\geq kT}}
                \\
                \ceq[iii]& \frac{1}{2}\prob<\distr<\rvector{x}>>{\innerprod{\cvector{u}}{\rvector{x}}\geq kT} + \frac{1}{2}\sbr{\int_{kT}^{+\infty}\prob<\rscalar{z}\sim\gaussian[0][\sigma^2]>{\rscalar{z} + \cscalar{t}\in\subsets<T>} \funcsbr{\phi}[\cscalar{t}]d\cscalar{t} - \int_{kT}^{+\infty}\prob<\rscalar{z}\sim\gaussian[0][\sigma^2]>{\rscalar{z} + \cscalar{t}\in\subsets*<T>} \funcsbr{\phi}[\cscalar{t}]d\cscalar{t}}
                \\
                =&\frac{1}{2}\prob<\distr<\rvector{x}>>{\innerprod{\cvector{u}}{\rvector{x}}\geq kT} + \frac{1}{2}\int_{kT}^{+\infty}\sbr{\int_{-\infty}^{+\infty}\sbr{\indicator[\cscalar{u} + \cscalar{t}\in\subsets<T>] - \indicator[\cscalar{u} + \cscalar{t}\in\subsets*<T>]}\funcsbr{\phi}<0, \sigma>[\cscalar{u}]d\cscalar{u}}\funcsbr{\phi}[\cscalar{t}]d\cscalar{t}
                \\
                =& \frac{1}{2}\prob<\distr<\rvector{x}>>{\innerprod{\cvector{u}}{\rvector{x}}\geq kT} + \frac{1}{2}\int_{kT}^{+\infty}\expect*<\rscalar{z}\sim\gaussian[0][\sigma^2]>{\funcsbr*{\sgn}[\funcsbr*{\sin}[\frac{2\pi(\rscalar{z} + \cscalar{t})}{T}]]}\funcsbr{\phi}[\cscalar{t}]d\cscalar{t}
                \\
                \geq& \frac{1}{2}\prob<\distr<\rvector{x}>>{\innerprod{\cvector{u}}{\rvector{x}}\geq kT} + \frac{T\cscalar{e}|-k^2T^2 - 2\pi^2\sigma^2/T^2|}{4\pi^2}\label{eq:difference-between-two-alternating-bands}
            \end{align*}
            where Equation (i) holds because $\rscalar{z}$ and $\rvector{x}$ are independent from each other by our assumption; Equation (ii) is obtained by using the \emph{law of total probability}; Equation (iii) holds by observing that $\innerprod{\cvector{u}}{\rvector{x}}\sim\gaussian$ due to $\rvector{x}\sim\gaussian|d|$ and the symmetry of standard normal distributions; and the last inequality is obtained by applying \lemmaref{lma:integral-of-noisy-square-wave-on-gaussian-shift} with $\omega=2\pi/T$.
        \end{proof}
        An important trick above is that we reduced the difference of two disjoint sequences of \emph{bands} under Gaussian measure to the partial convolution between a square wave and Gaussian density. This reformulation enabled us to apply Fourier analysis on the square wave function, which eventually led to the lower bound used in the above proof. We give the formal analysis in the following lemma.
        \begin{lemma}[Partial Gaussian-Square Wave Convolution Lower Bound]
        \label{lma:integral-of-noisy-square-wave-on-gaussian-shift}
            For any $\alpha > 0$ and $\omega \geq 2$ such that $\alpha\omega = 2k\pi$ for some $k\in\integer$, it holds that 
            \begin{equation*}
                \int_\alpha^{+\infty}\expect<\rscalar{z}\sim\gaussian[0][\sigma^2]>{\funcsbr{f}[\omega\sbr{\rscalar{z} + \cscalar{t}}]}\funcsbr{\phi}[\cscalar{t}]d\cscalar{t} \geq \frac{\cscalar{e}|- \alpha^2 - \sigma^2\omega^2/2|}{\pi\omega}
            \end{equation*}
            where $f:\real \rightarrow \binarydomain$ is the square wave function such that $\funcsbr{f}[u] = \funcsbr{\sgn}[\funcsbr{\sin}[u]]$.
        \end{lemma}
        \begin{proof}
            Notice that, due to the linearity of Gaussian distributions, $\rscalar{z} + \cscalar{t}\sim \gaussian[\cscalar{t}][\sigma^2]$ for any fixed $\cscalar{t}\in\real$. Therefore, we have that
            \begin{align}
                 \expect<\rscalar{z}\sim\gaussian[0][\sigma^2]>{\funcsbr{f}[\omega\sbr{\rscalar{z} + \cscalar{t}}]} =& \expect<\rscalar{z}\sim\gaussian[\cscalar{t}][\sigma^2]>{\funcsbr{f}[\omega\rscalar{z}]}
                 \notag
                 \\
                 \ceq[i]& \int_{-\infty}^{+\infty} \frac{4}{\pi}\sum_{k=1}^{+\infty}\frac{\funcsbr{\sin}[\omega\sbr{2k-1} \cscalar{u}]}{2k - 1} \funcsbr{\phi}<t, \sigma>[\cscalar{u}] d\cscalar{u}
                \notag
                \\
                =& \sum_{k=1}^{+\infty}\frac{4}{\pi(2k - 1)}\expect<\rscalar{z}\sim\gaussian[\cscalar{t}][\sigma^2]>{\funcsbr{\sin}[\omega\sbr{2k-1} \rscalar{z}]}
                \notag
                \\
                =& \frac{4}{\pi}\sum_{k=1}^{+\infty}\frac{\funcsbr{\sin}[\omega\sbr{2k-1} \cscalar{t}]\cscalar{e}|-\sigma^2\omega^2\sbr{2k-1}^2/2|}{(2k - 1)}
                \label{eq:gaussian-expectation-of-shifted-square-wave-func}
            \end{align}
            where Equation (i) holds due to \cref{fac:fourier-series-of-square-wave}, and the last equation is obtained by applying \lemmaref{lma:sub-gaussian-fourier-transform}.

            Substituting Equation \eqref{eq:gaussian-expectation-of-shifted-square-wave-func} back in the target expression gives
            \begin{align*}
                \int_\alpha^{+\infty}\expect<\rscalar{z}\sim\gaussian[0][\sigma^2]>{\funcsbr{f}[\omega\sbr{\rscalar{z} + \cscalar{t}}]}\funcsbr{\phi}[\cscalar{t}]d\cscalar{t} =& \frac{4}{\pi}\sum_{k=1}^{+\infty} \frac{\cscalar{e}|-\sigma^2\omega^2\sbr{2k-1}^2/2|}{(2k - 1)} \int_\alpha^{+\infty}\funcsbr{\sin}[\omega\sbr{2k-1} \cscalar{t}]\funcsbr{\phi}[\cscalar{t}]d\cscalar{t}
                \\
                \cgeq[i]& \frac{\cscalar{e}|- \alpha^2|}{\pi\omega}\sum_{k=1}^{+\infty} \frac{\cscalar{e}|-\sigma^2\omega^2\sbr{2k-1}^2/2|}{\sbr{2k - 1}|2|} 
                \\
                \geq& \frac{\cscalar{e}|- \alpha^2 - \sigma^2\omega^2/2|}{\pi\omega}
            \end{align*}
            where Inequality (i) is obtained by applying \lemmaref{lma:partial-fourier-transform-of-gaussian-density-start-from-alpha}, and the last inequality follows by only keeping the $k = 1$ term in the summation.
        \end{proof}

    \section{Hardness of Identifying Special Halfspaces Under Gaussian Marginals}
    \label{sec:hardness-of-identifying-special-halfspaces-under-gaussian}
        In this section, we present our main hardness results for all three problems---positive-reliable learning of homogeneous halfspaces, agnostically learning homogeneous halfspaces, and fairness auditing over homogeneous halfspace subgroups---under Gaussian marginals. Our approach follows a unified strategy across all three problems. For each problem, we proceed in two steps:
        
        \textbf{Indistinguishability:} We first establish that distinguishing between a distribution $\distr$ of $(\rvector{x}, \rscalar{y})$, where a homogeneous halfspace with the desired special property exists, and a distribution $\distr$ of $(\rvector{x}, \rscalar{y})$, where no such homogeneous halfspace exists, is computationally hard. This is achieved through a reduction from the cLWE problem $\text{LWE}(\cscalar{d}|\bigO{\cscalar{\eta}|\alpha|}|, \gaussian|d|, \sphere|d-1|,$ $\gaussian[0][\cscalar{\sigma}|2|], \modulo<T>)$. Informally speaking, we efficiently convert the labels $\rscalar{y '}\in \real<q>$ from the cLWE instance to $\rscalar{y}\in\binarydomain$ such that a homogeneous halfspace with the desired property exists if $\rscalar{y '}$ comes from the \emph{alternative hypothesis}, while no such homogeneous halfspace can exist if $\rscalar{y '}$ comes from the \emph{null hypothesis}.
        
        \textbf{Inapproximability:} We then derive the approximation lower bound as a corollary by showing that any algorithm that achieves the assumed approximation guarantee could be used to distinguish between the alternative and null hypotheses, contradicting the indistinguishability property.

        In the following section, we describe our reduction and formalize the above arguments.

        \subsection{Positive-Reliable Learning of Homogeneous Halfspaces}
        \label{sec:positive-reliable-learning-of-homogeneous-halfspaces}
            
            In this subsection, we establish the cryptographic hardness of agnostically learning one-sided homogeneous halfspaces under Gaussian marginals (c.f. \cref{prob:positive-reliable-learning}). 

            \begin{theorem}[Indistinguishability of One-sided Homogeneous Halfspaces]
            \label{thm:indistinguishability-of-one-sided-homogeneous-halfspaces}
                Under \cref{asp:sub-exponential-assumption-of-lwe}, for any $d\in\naturals$, any constants $\alpha\in(0, 1), \gamma\in\real<+>$, and any $\log^{\gamma} d\leq \cscalar{\eta} \leq \cscalar{c}d$ where $\cscalar{c}$ is a sufficiently small constant, there is no algorithm that runs in time $\cscalar{d}|\bigO{\cscalar{\eta}|\alpha|}|$ and distinguishes between the following two cases of a joint distribution $\distr$ of $(\rvector{x}, \rscalar{y})$ supported on $\real|d|\times\binarydomain$ such that the marginal $\distr<\rvector{x}> = \gaussian|d|$, with $\cscalar{d}|-\bigO{\cscalar{\eta}|\alpha|} |$ advantage:
                \begin{itemize}
                    \item[\textbf{Alternative Hypothesis}:]$\exists\cvector{v}\in \sphere|d-1|$ such that $\prob<(\rvector{x}, \rscalar{y})\sim\distr>{\rscalar{y} = 1\cond\innerprod{\cvector{v}}{\rvector{x}} \geq 0} = \frac{1}{2} + \bigM{\frac{1}{\sqrt{\eta\log d}}}$.
                    \item[\textbf{Null Hypothesis}:]$\forall\cvector{u}\in \sphere|d-1|$, it holds that $\prob<(\rvector{x}, \rscalar{y})\sim\distr>{\rscalar{y} = 1\cond\innerprod{\cvector{u}}{\rvector{x}} \geq 0} = \frac{1}{2}$.
                \end{itemize}
            \end{theorem}
            \begin{proof}
                We give an efficient method taking as input samples $(\rvector{x}, \rscalar{y})\sim\distr '$ either from the alternative or the null hypothesis of $\text{LWE}(\cscalar{d}|\bigO{\cscalar{\eta}|\alpha|}|, \gaussian|d|, \sphere|d-1|,$ $\gaussian[0][\cscalar{\sigma}|2|], \modulo<T>)$ as in \lemmaref{lma:sub-exponential-hardness-of-clwe}, and generating samples from another distribution $\distr$ with the following properties: if $\distr '$ is from the alternative (resp.\ null) hypothesis of the LWE problem, then the resulting distribution $\distr$ will satisfy the alternative (resp.\ null) hypothesis requirement of the theorem statement.
            
                The reduction process is as follows: for $(\rvector{x}, \rscalar{y '})$ sampled from a instance $\distr'$ of the problem LWE$(\cscalar{d}|\bigO{\cscalar{\eta}|\alpha|}|,\gaussian|d|, \sphere|d-1|, \gaussian[0][\sigma^2],\modulo<T>)$ as in \lemmaref{lma:sub-exponential-hardness-of-clwe}, we output $(\rvector{x}, \rscalar{y})\sim\distr$, such that
                \begin{equation*}
                    \rscalar{y}=
                    \begin{cases}
                        +1,\quad \text{if }\rscalar{y'}\leq T/2
                        \\
                        -1,\quad \text{otherwise}.
                    \end{cases}
                \end{equation*}
                We argue that $\distr$ from the Alternative Hypothesis satisfies the stated properties with $\cvector{v} = \cvector{s}$:
            
                \textbf{For the alternative hypothesis case}, by the definition of $\rscalar{y'}$ (c.f.~\definitionref{def:learning-with-errors}), we refer to \propositionref{prop:lower-bound-of-the-probability-mass-of-alternating-bands-in-halfspace} with $\cvector{u} = \cvector{s}$ and $k = 0$ to obtain
                \begin{align}
                     \prob<(\rvector{x}, \rscalar{y})\sim\distr>{\rscalar{y} = 1\cap\innerprod{\cvector{s}}{\rvector{x}} \geq 0} \geq& \frac{1}{2}\prob<\rvector{x}\sim\distr<\rvector{x}>>{\innerprod{\cvector{s}}{\rvector{x}}\geq 0} + \frac{T\cscalar{e}|- 2\pi^2\sigma^2/T^2|}{4\pi^2}
                     \notag
                     \\
                     =& \frac{1}{4} + \frac{T\cscalar{e}|- 2\pi^2\sigma^2/T^2|}{4\pi^2}\label{eq:lower-bound-of-the-joint-event-for-one-sided-homogeneous-halfspaces}
                \end{align}
                where the last equation holds because $\rvector{x}\sim\gaussian|d|$ and $\cvector{s}\in\sphere|d-1|$ implies $\innerprod{\cvector{s}}{\rvector{x}}\sim\gaussian$.
    
                Recall that \lemmaref{lma:sub-exponential-hardness-of-clwe} states that the LWE problem is hard for any fixed constant $\kappa\in \naturals$ and $\sigma\geq \cscalar{\eta}|-\kappa|$. Given the constant $\gamma\in \real<+>$ in this theorem, we fix $\kappa = \ceil{1/2\gamma + 1/2}$. Then, by \lemmaref{lma:sub-exponential-hardness-of-clwe}, the LWE problem is hard for $\sigma = \cscalar{\eta}|-\kappa|\leq 1/(\sqrt{\cscalar{\eta}\log d}) = O(T)$. Thus, substituting $\sigma = \cscalar{\eta}|-\kappa|$ back in Equation \eqref{eq:lower-bound-of-the-joint-event-for-one-sided-homogeneous-halfspaces} with $T = 1/C'\sqrt{\eta\log d}$ (c.f.~\lemmaref{lma:sub-exponential-hardness-of-clwe}) gives $\prob<\distr>{\rscalar{y} = 1\cap\innerprod{\cvector{s}}{\rvector{x}} \geq 0} = \frac{1}{4} + \bigM{\frac{1}{\sqrt{\eta\log d}}}$, which implies $\prob<\distr>{\rscalar{y} = 1\cond\innerprod{\cvector{s}}{\rvector{x}} \geq 0} = \frac{1}{2} + \bigM{\frac{1}{\sqrt{\eta\log d}}}$.
                
                \textbf{For the null hypothesis}, because $\rscalar{y}$ is generated independently at random, half of the data points will be mislabeled regardless of which halfspace we choose. Therefore, for all $\cvector{u}\in\sphere|d-1|$, it holds that $\prob<\distr>{\rscalar{y} = 1\cond\innerprod{\cvector{u}}{\rvector{x}} \geq 0} = 1/2$.
            
                Finally, we verify the time complexity and the distinguishing advantage for learning one-sided homogeneous halfspaces. By \lemmaref{lma:sub-exponential-hardness-of-clwe}, the problem LWE$(\cscalar{d}|\bigO{\cscalar{\eta}|\alpha|}|, \gaussian|d|, \sphere|d-1|, \gaussian[0][\sigma^2], \modulo<T>)$ cannot be solved in $\cscalar{d}|\bigO{\cscalar{\eta}|\alpha|}|$ time with $\cscalar{d}|-\bigO{\cscalar{\eta}|\alpha|}|$ advantage with any choice of $\alpha\in(0,1)$, $\sigma\geq \cscalar{\eta}|-\kappa|$ (where $\kappa\in \naturals$ is a constant) and $T = 1/C'\sqrt{\cscalar{\eta}\log d}$ (where $C' > 0$ is a sufficiently large constant), under \cref{asp:sub-exponential-assumption-of-lwe}. Therefore, under the same assumption, no algorithm can solve the decision problem as in the statement in $\cscalar{d}|\bigO{\cscalar{\eta}|\alpha|}|$ time with $\cscalar{d}|-\bigO{\cscalar{\eta}|\alpha|}|$ advantage.
            \end{proof}
            
            Thus, any algorithm distinguishing the two cases in \cref{thm:indistinguishability-of-one-sided-homogeneous-halfspaces} could distinguish the cLWE hypotheses, which contradicts the theorem. We state the main hardness result for \cref{prob:positive-reliable-learning} as follows, but postpone the formal proof to \cref{sec:omitted-proofs-of-positive-reliable-learning} due to the page limit.
            
            \begin{corollary}[\cref{thm:hardness-of-positive-reliable-learning-over-homogeneous-halfspace}]
            \label{cor:hardness-of-positive-reliable-learning-over-homogeneous-halfspace}
                Given \cref{asp:sub-exponential-assumption-of-lwe}, for any constants $\alpha\in(0, 1), \gamma > 1/2$, and any $c/\sqrt{d\log d}\leq \epsilon\leq 1/\log^\gamma d$ where $c$ is a sufficiently small constant, there is no $(1, \epsilon)$-approximation algorithm that solves \cref{prob:positive-reliable-learning} and runs in time $\cscalar{d}|\bigO{1/\sbr{\cscalar{\epsilon}\sqrt{\log d}}|2\alpha|}|$.
            \end{corollary}
            
        \subsection{Agnostically Learning Homogeneous Halfspaces}
        \label{sec:agnostically-learning-homogeneous-halfspaces}
            In this section, we establish the cryptographic hardness of agnostically learning homogeneous halfspaces under Gaussian marginals (c.f. \cref{prob:agnostic-learning}). The details are similar, and we defer the complete proofs to \cref{sec:omitted-proofs-of-agnostic-learning}.

            \begin{theorem}[Indistinguishability of Homogeneous Halfspaces]
            \label{thm:indistinguishability-of-agnostic-learning}
                Under \cref{asp:sub-exponential-assumption-of-lwe}, for any $d\in\naturals$, any constants $\alpha\in(0, 1), \gamma\in\real<+>$, and any $\log^{\gamma} d\leq \cscalar{\eta} \leq \cscalar{c}d$ where $\cscalar{c}$ is a sufficiently small constant, there is no algorithm that runs in time $\cscalar{d}|\bigO{\cscalar{\eta}|\alpha|}|$ and distinguishes between the following two cases of a joint distribution $\distr$ of $(\rvector{x}, \rscalar{y})$ supported on $\real|d|\times\binarydomain$ such that the marginal $\distr<\rvector{x}> = \gaussian|d|$, with $\cscalar{d}|-\bigO{\cscalar{\eta}|\alpha|} |$ advantage:
                \begin{itemize}
                    \item[\textbf{Alternative Hypothesis}:]$\exists\cvector{v}\in \sphere|d-1|$ such that $\prob<(\rvector{x}, \rscalar{y})\sim\distr>{\rscalar{y} = \funcsbr{\sgn}[\innerprod{\cvector{v}}{\rvector{x}}]} = \frac{1}{2} + \bigM{\frac{1}{\sqrt{\eta\log d}}}$.
                    \item[\textbf{Null Hypothesis}:]$\forall\cvector{u}\in \sphere|d-1|$, it holds that $\prob<(\rvector{x}, \rscalar{y})\sim\distr>{\rscalar{y} = \funcsbr{\sgn}[\innerprod{\cvector{u}}{\rvector{x}}]} = \frac{1}{2}$.
                \end{itemize}
            \end{theorem}

            The proof idea is generally the same as that of \cref{thm:indistinguishability-of-one-sided-homogeneous-halfspaces}, except that we now need to deal with both \emph{sides} of the defining hyperplane, i.e., 
            \begin{equation*}
                \prob<(\rvector{x}, \rscalar{y})\sim\distr>{\rscalar{y} = \funcsbr{\sgn}[\innerprod{\cvector{v}}{\rvector{x}}]} = \prob<(\rvector{x}, \rscalar{y})\sim\distr>{\rscalar{y} = 1\cap\innerprod{\cvector{v}}{\rvector{x}} \geq 0} + \prob<(\rvector{x}, \rscalar{y})\sim\distr>{\rscalar{y} = -1\cap\innerprod{\cvector{v}}{\rvector{x}} < 0},
            \end{equation*}
            where $\prob{\rscalar{y} = 1\cap\innerprod{\cvector{v}}{\rvector{x}} \geq 0}$ is the objective function in positive-reliable learning (c.f. \cref{prob:positive-reliable-learning}). It turns out that we can easily prove $\prob{\rscalar{y} = 1\cap\innerprod{\cvector{v}}{\rvector{x}} \geq 0} = \prob{\rscalar{y} = -1\cap\innerprod{\cvector{v}}{\rvector{x}} < 0}$ by the symmetry of isotropic Gaussian distributions. The rest of the proof follows as in \cref{thm:indistinguishability-of-one-sided-homogeneous-halfspaces}.

            Notice that, even though the target objective in \cref{thm:indistinguishability-of-agnostic-learning} is different from the classification loss, $\prob<\distr>{\rscalar{y} \neq \funcsbr{\sgn}[\innerprod{\cvector{v}}{\rvector{x}}]}$, defined in \cref{prob:agnostic-learning}, they can be easily transformed to each other by negating the labels $\rscalar{y}\sim\distr<\rscalar{y}>$. Therefore, \cref{thm:indistinguishability-of-agnostic-learning} similarly implies the hardness of \cref{prob:agnostic-learning}.

            \begin{corollary}[\cref{thm:hardness-of-learning-homogeneous-halfspace-under-gaussian}]
            \label{cor:hardness-of-learning-homogeneous-halfspace-under-gaussian}
                Given \cref{asp:sub-exponential-assumption-of-lwe}, for any constants $\alpha\in(0, 1), \gamma > 1/2$, and any $c/\sqrt{d\log d}\leq \epsilon\leq 1/\log^\gamma d$ where $c$ is a sufficiently small constant, there is no $(1, \epsilon)$-approximation algorithm that solve \cref{prob:agnostic-learning} and runs in time $\cscalar{d}|\bigO{1/\sbr{\cscalar{\epsilon}\sqrt{\log d}}|2\alpha|}|$.
            \end{corollary}

        \subsection{Fairness Auditing Over Homogeneous Halfspace Subgroups}
        \label{sec:fairness-auditing-over-homogeneous-halfspace-subgroups}
            Finally, we establish the cryptographic hardness of auditing SPSF (c.f.~\definitionref{def:statistical-parity-subgroup-fairness}) over homogeneous halfspace subgroups under Gaussian marginals (c.f.~\cref{prob:fairness-auditing}). Different from the previous results, we will show strong inapproximability of the auditing problem in multiplicative form, in addition to the additive form. Due to their similarity to \cref{thm:indistinguishability-of-one-sided-homogeneous-halfspaces} and \corollaryref{cor:hardness-of-positive-reliable-learning-over-homogeneous-halfspace}, the formal proofs will be deferred to \cref{sec:omitted-proofs-of-fairness-auditing}.

            \begin{theorem}[Indistinguishability of Homogeneous Halfspace Subgroups]
            \label{thm:indistinguishability-of-fairness-auditing}
                Under \cref{asp:sub-exponential-assumption-of-lwe}, for any $d\in\naturals$, any constants $\alpha\in(0, 1), \gamma\in\real<+>$, and any $\log^{\gamma} d\leq \cscalar{\eta} \leq \cscalar{c}d$ where $\cscalar{c}$ is a sufficiently small constant, there is no algorithm that runs in time $\cscalar{d}|\bigO{\cscalar{\eta}|\alpha|}|$ and distinguishes between the following two cases of a joint distribution $\distr$ of $(\rvector{x}, \rscalar{y})$ supported on $\real|d|\times\binarydomain$ such that the marginal $\distr<\rvector{x}> = \gaussian|d|$, with $\cscalar{d}|-\bigO{\cscalar{\eta}|\alpha|} |$ advantage:
                \begin{itemize}
                    \item[\textbf{Alternative Hypothesis}:]$\exists\cvector{v}\in \sphere|d-1|$ such that $\funcsbr{u}<\distr>[\cvector{v}] = \bigM{1/\sqrt{\eta\log d}}$.
                    \item[\textbf{Null Hypothesis}:]$\forall\cvector{u}\in \sphere|d-1|$, it holds that $\funcsbr{u}<\distr>[\cvector{u}] = 0$.
                \end{itemize}
            \end{theorem}
            
            Our proof idea is, again, to bridge the level of unfairness $\funcsbr{u}<\distr>[\cvector{u}]$ with the objective function of learning one-sided homogeneous halfspaces, i.e., $\prob{\rscalar{y} = 1\cond\innerprod{\cvector{v}}{\rvector{x}} \geq 0}$. Recall in the proof of \cref{thm:indistinguishability-of-one-sided-homogeneous-halfspaces}, we have established a lower bound on $\prob{\rscalar{y} = 1\cond\innerprod{\cvector{s}}{\rvector{x}} \geq 0}$, which is a key term of $\funcsbr{u}<\distr>[\cvector{s}]$ (c.f. \definitionref{def:statistical-parity-subgroup-fairness}). We know that $\prob<\distr<\rvector{x}>>{\innerprod{\cvector{s}}{\rvector{x}} > 0} = 1/2$ because $\distr<\rvector{x}> = \gaussian|d|$. Furthermore, \citet{hsu2024distribution} showed $\prob<\distr<\rscalar{y}>>{\rscalar{y} = 1} = 1/2$. Observe now that $2\funcsbr{u}<\distr>[\cvector{s}] = \abs{1/2 - \prob{\rscalar{y} = 1\cond \innerprod{\cvector{s}}{\rvector{x}} > 0}}$. The rest of the proof then follows \cref{thm:indistinguishability-of-one-sided-homogeneous-halfspaces}.

            Similarly, \cref{thm:indistinguishability-of-fairness-auditing} implies the inapproximability of large unfairness over homogeneous halfspace subgroups in both additive and multiplicative forms (see \cref{sec:omitted-proofs-of-fairness-auditing} for proofs):

            \begin{corollary}[\cref{thm:hardness-of-auditing-halfspace-under-gaussian-in-additive-form}]
            \label{cor:hardness-of-auditing-halfspace-under-gaussian-in-additive-form}
                Given \cref{asp:sub-exponential-assumption-of-lwe}, for any constants $\alpha\in(0, 1), \gamma > 1/2$, and any $c/\sqrt{d\log d}\leq \epsilon\leq 1/\log^\gamma d$ where $c$ is a sufficiently small constant, there is no $(1, \epsilon)$-approximation algorithm that solves \cref{prob:fairness-auditing} and runs in time $\cscalar{d}|\bigO{1/\sbr{\cscalar{\epsilon}\sqrt{\log d}}|2\alpha|}|$.
            \end{corollary}
            
            \begin{corollary}[\cref{thm:hardness-of-auditing-halfspace-under-gaussian-in-multiplicative-form}]\label{cor:hardness-of-auditing-halfspace-under-gaussian-in-multiplicative-form}
                Given \cref{asp:sub-exponential-assumption-of-lwe}, there is no $1/\poly[d]$-approximation algorithm that solves \cref{prob:fairness-auditing} and runs in time $\poly[d]$.
            \end{corollary}
            
    \section{Future Work}
        One of the most natural questions arising from our findings is how much harder it is to identify general halfspaces than homogeneous halfspaces, especially for agnostic learning and positive-reliable learning over halfspaces. Superficially, the lack of upper bounds for learning general halfspaces seems to imply that the problem is more challenging. However, we fall short of proving a stronger lower bound for these problems through reduction from the cLWE problem. Looking at the problem of positive-reliable halfspace learning, one might conjecture that a stronger lower bound could exist if we impose a larger threshold lower bound on the halfspaces, e.g., no polynomial-time $(1, 1/\cscalar{\log}|(1/2 + c)/(1 + \tau)| d)$-approximation algorithm exists for halfspaces with threshold $t \geq \tau \geq 0$. This conjecture is based on the observation that learning halfspaces of smaller population size requires the ability to distinguish between halfspaces with smaller advantage. Another interesting direction is to investigate whether any multiplicative lower bound exists for learning halfspaces or one-sided halfspaces. Given the fact that no lower bound of multiplicative form is known for these problems, any non-trivial multiplicative bound could be a breakthrough for these problems.

% Acknowledgments---Will not appear in anonymized version
% \acks{We thank a bunch of people and funding agency.}

\bibliographystyle{abbrvnat}
\bibliography{refs}

\newpage
\appendix

% \crefalias{section}{appendix} % uncomment if you are using cleveref
    \input{appendix}

\end{document}

%% file: project-commands.tex
\usepackage{algorithm}
\usepackage{algorithmic}
\newcommand\algorithmicprocedure{\textbf{procedure}}
\newcommand{\algorithmicendprocedure}{\algorithmicend\ \algorithmicprocedure}
\makeatletter
\newcommand\PROCEDURE[3][default]{%
  \ALC@it
  \algorithmicprocedure\ \textsc{#2}(#3)%
  \ALC@com{#1}%
  \begin{ALC@prc}%
}
\newcommand\ENDPROCEDURE{%
  \end{ALC@prc}%
  \ifthenelse{\boolean{ALC@noend}}{}{%
    \ALC@it\algorithmicendprocedure
  }%
}
\newenvironment{ALC@prc}{\begin{ALC@g}}{\end{ALC@g}}
\makeatother

\input{basic-commands}

\NewDocumentCommand{\opt}{d<>d()d||}{
    \ScriptedMathSymbol{\mathrm{opt}}<#1>(#2)|#3|
}

\NewDocumentCommand{\conceptclass}{d<>d()d||}{
    \ScriptedMathSymbol{\mathcal{C}}<#1>(#2)|#3|
}

\NewDocumentCommand{\hypothesisclass}{d<>d()d||}{
    \ScriptedMathSymbol{\mathcal{H}}<#1>(#2)|#3|
}
\NewDocumentCommand{\halfspace}{sd<>d()d||o}{
    \IfBooleanTF{#1}{
        \funcsbr{h}<#2>(#3)|#4c|[#5]
    }{
        \funcsbr{h}<#2>(#3)|#4|[#5]
    }
}
\NewDocumentCommand{\parameterset}{md<>d()d||}{
    \ScriptedMathSymbol{\mathcal{#1}}<#2>(#3)|#4|
}

\NewDocumentCommand{\subsets}{sd<>d()D||{}O{}}{
    \IfBooleanTF{#1}{
        \IfNoValueOrEmptyTF{#5}{
            \ScriptedMathSymbol{S}<#2>(#3)|{#4}c|
        }{
            \ScriptedMathSymbol{S}<#2>(#3)|{#4}c|\sbr{#5}
        }
    }{
        \IfNoValueOrEmptyTF{#5}{
            \ScriptedMathSymbol{S}<#2>(#3)|#4|
        }{
            \ScriptedMathSymbol{S}<#2>(#3)|#4|\sbr{#5}
        }
    }
}

\NewDocumentCommand{\hypothesis}{sd<>d()o}{
    \IfBooleanTF{#1}{
        \funcsbr{h}<#2>(#3)|c|[#4]
    }{
        \funcsbr{h}<#2>(#3)[#4]
    }
}

\NewDocumentCommand{\algo}{d<>d()o}{
    \funcsbr{\mathcal{A}}<#1>(#2)[#3]
}

\NewDocumentCommand{\errorregion}{d<>d()d||o}{
    \funclbr{\mathcal{E}}<#1>(#2)|#3|[#4]
}

\NewDocumentCommand{\loss}{sd<>d()o}{
    \IfBooleanTF{#1}{
        \funcsbr*{\mathcal{L}}<#2>(#3)[#4]
    }{
        \funcsbr{\mathcal{L}}<#2>(#3)[#4]
    }
}

\NewDocumentCommand{\err}{sd<>d()o}{
    \IfBooleanTF{#1}{
        \funcsbr*{\mathrm{err}}<#2>(#3)[#4]
    }{
        \funcsbr{\mathrm{err}}<#2>(#3)[#4]
    }
}

\NewDocumentCommand{\surrloss}{sD<>{\sigma}d()D||{}o}{
    \IfBooleanTF{#1}{
        \funcsbr*{S'}<#2>(#3)|#4|[#5]
    }{
        \funcsbr*{S}<#2>(#3)|#4|[#5]
    }
}

\newcommand*{\propositionrefname}{Proposition}

\newcommand*{\lemmarefname}{Lemma}

\newcommand*{\corollaryrefname}{Corollary}

\newcommand*{\definitionrefname}{Definition}

\newcommand*{\propositionref}[1]{%
  {\propositionrefname}~{\ref{#1}}}
\newcommand*{\lemmaref}[1]{%
  {\lemmarefname}~{\ref{#1}}}

\newcommand*{\corollaryref}[1]{%
  {\corollaryrefname}~{\ref{#1}}}
\newcommand*{\definitionref}[1]{%
  {\definitionrefname}~{\ref{#1}}}

%% file: basic-commands.tex
\usepackage[utf8]{inputenc} % allow utf-8 input
\usepackage[T1]{fontenc}    % use 8-bit T1 fonts
\usepackage{url}            % simple URL typesetting
\usepackage{paralist}
\usepackage{booktabs}       % professional-quality tables
\usepackage{amsfonts}       % blackboard math symbols
\usepackage{amsthm}
\usepackage{amsmath}
\usepackage{bbm}
\usepackage{amssymb}
\usepackage{bm}
\usepackage{thmtools}
\usepackage{dsfont}
\usepackage{nicefrac}       % compact symbols for 1/2, etc.
\usepackage{microtype}      % microtypography
\usepackage{lipsum}
\usepackage{fancyhdr}       % header
\usepackage{graphicx,adjustbox}       % graphics
\usepackage{xparse}         % commands
\usepackage{xcolor}
\usepackage{paracol}
\usepackage{hyperref}
\usepackage[capitalize,noabbrev]{cleveref}
\ExplSyntaxOn
\DeclareExpandableDocumentCommand{\IfNoValueOrEmptyTF}{mmm}
 {
  \IfNoValueTF{#1}{#2}
   {
    \tl_if_empty:nTF {#1} {#2} {#3}
   }
 }
\ExplSyntaxOff

\NewDocumentCommand{\oversetwo}{om}{
    \IfNoValueOrEmptyTF{#1}{
        {#2}
    }{
        \overset{(\mathrm{#1})}{#2}
    }
}
\NewDocumentCommand{\cleq}{o}{
    \oversetwo[#1]{\leq}
}
\NewDocumentCommand{\cgeq}{o}{
    \oversetwo[#1]{\geq}
}
\NewDocumentCommand{\cl}{o}{
    \oversetwo[#1]{<}
}
\NewDocumentCommand{\cg}{o}{
    \oversetwo[#1]{>}
}
\NewDocumentCommand{\ceq}{o}{
    \oversetwo[#1]{=}
}

\NewDocumentCommand{\ScriptedMathSymbol}{md<>d()d||}{
    \IfNoValueOrEmptyTF{#2}{
        \IfNoValueOrEmptyTF{#3}{
            \IfNoValueOrEmptyTF{#4}{
                #1
            }{
                #1^{#4}
            }
        }{
            \IfNoValueOrEmptyTF{#4}{
                #1^{(#3)}
            }{
                #1^{(#3)#4}
            }
        }
    }{
        \IfNoValueOrEmptyTF{#3}{
            \IfNoValueOrEmptyTF{#4}{
                #1_{#2}
            }{
                #1_{#2}^{#4}
            }
        }{
            \IfNoValueOrEmptyTF{#4}{
                #1_{#2}^{(#3)}
            }{
                #1_{#2}^{(#3)#4}
            }
        }
    }
}
\NewDocumentCommand{\rvector}{smd<>d()d||}{
    \IfBooleanTF{#1}{
        {\ScriptedMathSymbol{\bar{\mathbf{#2}}}<#3>(#4)|#5|}
    }{
        {\ScriptedMathSymbol{\mathbf{#2}}<#3>(#4)|#5|}
    }
}

\NewDocumentCommand{\cvector}{smd<>d()d||}{
    \IfBooleanTF{#1}{
        {\ScriptedMathSymbol{\bar{\bm{#2}}}<#3>(#4)|#5|}
    }{
        {\ScriptedMathSymbol{\bm{#2}}<#3>(#4)|#5|}
    }
}
\NewDocumentCommand{\cmatrix}{md<>d()d||}{
    \cvector{#1}<#2>(#3)|#4|
}
\NewDocumentCommand{\rmatrix}{md<>d()d||}{
    \rvector{\mathbf{#1}}<#2>(#3)|#4|
}

\NewDocumentCommand{\rvectorseq}{smO{1}d<>O{1}d()d||}{
    \IfBooleanTF{#1}{
        \IfNoValueOrEmptyTF{#4}{
            \rvector*{#2}(#5)|#7|, \ldots, \rvector*{#2}(#6)|#7|
        }{
            \rvector*{#2}<#3>(#6)|#7|, \ldots, \rvector*{#2}<#4>(#6)|#7|
        }  
    }{
        \IfNoValueOrEmptyTF{#4}{
            \rvector{#2}(#5)|#7|, \ldots, \rvector{#2}(#6)|#7|
        }{
            \rvector{#2}<#3>(#6)|#7|, \ldots, \rvector{#2}<#4>(#6)|#7|
        }  
    }   
}
\NewDocumentCommand{\cvectorseq}{smO{1}d<>O{1}d()d||}{
    \IfBooleanTF{#1}{
        \IfNoValueOrEmptyTF{#4}{
            \cvector*{#2}(#5)|#7|, \ldots, \cvector*{#2}(#6)|#7|
        }{
            \cvector*{#2}<#3>(#6)|#7|, \ldots, \cvector*{#2}<#4>(#6)|#7|
        }  
    }{
        \IfNoValueOrEmptyTF{#4}{
            \cvector{#2}(#5)|#7|, \ldots, \cvector{#2}(#6)|#7|
        }{
            \cvector{#2}<#3>(#6)|#7|, \ldots, \cvector{#2}<#4>(#6)|#7|
        }  
    }   
}

\NewDocumentCommand{\rscalar}{smd<>d()d||}{
    \IfBooleanTF{#1}{
        {\ScriptedMathSymbol{\bar{\mathrm{#2}}}<#3>(#4)|#5|}
    }{
        {\ScriptedMathSymbol{\mathrm{#2}}<#3>(#4)|#5|}
    }
}
\NewDocumentCommand{\cscalar}{smd<>d()d||}{
    \IfBooleanTF{#1}{
        {\ScriptedMathSymbol{\bar{#2}}<#3>(#4)|#5|}
    }{
        {\ScriptedMathSymbol{#2}<#3>(#4)|#5|}
    }
}
\NewDocumentCommand{\rscalarseq}{smO{1}d<>O{1}d()d||}{
    \IfBooleanTF{#1}{
        \IfNoValueOrEmptyTF{#4}{
            \rscalar*{#2}(#5)|#7|, \ldots, \rscalar*{#2}(#6)|#7|
        }{
            \rscalar*{#2}<#3>(#6)|#7|, \ldots, \rscalar*{#2}<#4>(#6)|#7|
        }  
    }{
        \IfNoValueOrEmptyTF{#4}{
            \rscalar{#2}(#5)|#7|, \ldots, \rscalar{#2}(#6)|#7|
        }{
            \rscalar{#2}<#3>(#6)|#7|, \ldots, \rscalar{#2}<#4>(#6)|#7|
        }  
    }   
}
\NewDocumentCommand{\cscalarseq}{smO{1}d<>O{1}d()d||}{
    \IfBooleanTF{#1}{
        \IfNoValueOrEmptyTF{#4}{
            \cscalar*{#2}(#5)|#7|, \ldots, \cscalar*{#2}(#6)|#7|
        }{
            \cscalar*{#2}<#3>(#6)|#7|, \ldots, \cscalar*{#2}<#4>(#6)|#7|
        }  
    }{
        \IfNoValueOrEmptyTF{#4}{
            \cscalar{#2}(#5)|#7|, \ldots, \cscalar{#2}(#6)|#7|
        }{
            \cscalar{#2}<#3>(#6)|#7|, \ldots, \cscalar{#2}<#4>(#6)|#7|
        }  
    }   
}
\NewDocumentCommand{\sbr}{smd||}{
    \IfBooleanTF{#1}{
        \ensuremath \ScriptedMathSymbol{( #2 )}|#3|
    }{
        \ensuremath \ScriptedMathSymbol{\left( #2 \right)}|#3|
    }
}
\NewDocumentCommand{\mbr}{smd||}{
    \IfBooleanTF{#1}{
        \ensuremath \ScriptedMathSymbol{[ #2 ]}|#3|
    }{
        \ensuremath \ScriptedMathSymbol{\left[ #2 \right]}|#3|
    }
}
\NewDocumentCommand{\lbr}{smd||}{
    \IfBooleanTF{#1}{
        \ensuremath \ScriptedMathSymbol{\{ #2 \}}|#3|
    }{
        \ensuremath \ScriptedMathSymbol{\left\{ #2 \right\}}|#3|
    }
}

\NewDocumentCommand{\funcsbr}{smd<>d()d||o}{
    \IfBooleanTF{#1}{
        \IfNoValueOrEmptyTF{#6}{
            \ScriptedMathSymbol{#2}<#3>(#4)|#5|
        }{
            \ScriptedMathSymbol{#2}<#3>(#4)|#5|\sbr{#6}
        }
    }{
        \IfNoValueOrEmptyTF{#6}{
            \ScriptedMathSymbol{#2}<#3>(#4)|#5|
        }{
            \ScriptedMathSymbol{#2}<#3>(#4)|#5|\sbr*{#6}
        }
    }
}
\NewDocumentCommand{\funcmbr}{smd<>d()d||o}{
    \IfBooleanTF{#1}{
        \IfNoValueOrEmptyTF{#6}{
            \ScriptedMathSymbol{#2}<#3>(#4)|#5|
        }{
            \ScriptedMathSymbol{#2}<#3>(#4)|#5|\mbr{#6}
        }
    }{
        \IfNoValueOrEmptyTF{#6}{
            \ScriptedMathSymbol{#2}<#3>(#4)|#5|
        }{
            \ScriptedMathSymbol{#2}<#3>(#4)|#5|\mbr*{#6}
        }
    }
}
\NewDocumentCommand{\funclbr}{smd<>d()d||o}{
    \IfBooleanTF{#1}{
        \IfNoValueOrEmptyTF{#6}{
            \ScriptedMathSymbol{#2}<#3>(#4)|#5|
        }{
            \ScriptedMathSymbol{#2}<#3>(#4)|#5|\lbr{#6}
        }
    }{
        \IfNoValueOrEmptyTF{#6}{
            \ScriptedMathSymbol{#2}<#3>(#4)|#5|
        }{
            \ScriptedMathSymbol{#2}<#3>(#4)|#5|\lbr*{#6}
        }
    }
}
\NewDocumentCommand{\derivative}{d<>d||}{
    \ScriptedMathSymbol{\nabla}<#1>|#2|
}
\NewDocumentCommand{\funcclass}{md<>d()d||}{
    \ScriptedMathSymbol{\mathcal{#1}}<#2>(#3)|#4|
}

\DeclareMathOperator*{\sgn}{\mathrm{sgn}}

\NewDocumentCommand{\real}{d<>d||}{
    \ScriptedMathSymbol{\mathbb{R}}<#1>|#2|
}
\NewDocumentCommand{\complex}{d<>d||}{
    \ScriptedMathSymbol{\mathbb{C}}<#1>|#2|
}
\NewDocumentCommand{\naturals}{d<>d||}{
    \ScriptedMathSymbol{\mathbb{N}}<#1>|#2|
}
\NewDocumentCommand{\integer}{d<>d||}{
    \ScriptedMathSymbol{\mathbb{Z}}<#1>|#2|
}
\NewDocumentCommand{\sphere}{d<>d||}{
    \ScriptedMathSymbol{\mathbb{S}}<#1>|#2|
}
\NewDocumentCommand\booldomain{d||}{
    \ScriptedMathSymbol{\lbr{0 ,1}}|#1|
}
\NewDocumentCommand{\binarydomain}{d||}{
    \ScriptedMathSymbol{\lbr{\pm 1}}|#1|
}

\NewDocumentCommand{\union}{sd<>d||}{
    \IfBooleanTF{#1}{
        \IfNoValueOrEmptyTF{#2}{
            \ScriptedMathSymbol{\ \bigcup\ }<#2>|#3|
        }{
            \ScriptedMathSymbol{\bigcup}<#2>|#3|
        }     
    }{
        \IfNoValueOrEmptyTF{#2}{
            \ScriptedMathSymbol{\ \cup\ }<#2>|#3|
        }{
            \ScriptedMathSymbol{\cup}<#2>|#3|
        }
    }
}

\NewDocumentCommand{\isect}{sd<>d||}{
    \IfBooleanTF{#1}{
        \IfNoValueOrEmptyTF{#2}{
            \ScriptedMathSymbol{\ \bigcap\ }<#2>|#3|
        }{
            \ScriptedMathSymbol{\bigcap}<#2>|#3|
        }     
    }{
        \IfNoValueOrEmptyTF{#2}{
            \ScriptedMathSymbol{\ \cap\ }<#2>|#3|
        }{
            \ScriptedMathSymbol{\cap}<#2>|#3|
        }
    }
}

\NewDocumentCommand{\por}{sd<>d||}{
    \IfBooleanTF{#1}{
        \ScriptedMathSymbol{\bigvee}<#2>|#3|
    }{
        \ScriptedMathSymbol{\vee}<#2>|#3|
    }
}
\NewDocumentCommand{\pand}{sd<>d||}{
    \IfBooleanTF{#1}{
        \ScriptedMathSymbol{\bigwedge}<#2>|#3|
    }{
        \ScriptedMathSymbol{\wedge}<#2>|#3|
    }
}
\NewDocumentCommand{\cond}{s}{
    \IfBooleanTF{#1}{
        |
    }{
        {\ |\ }
    }
}

\NewDocumentCommand{\ceil}{sm}{
    \IfBooleanTF{#1}{
        \ensuremath \lceil {#2} \rceil
    }{
        \ensuremath \left\lceil {#2} \right\rceil
    }
}
\NewDocumentCommand{\floor}{sm}{
    \IfBooleanTF{#1}{
        \ensuremath \lfloor {#2} \rfloor
    }{
        \ensuremath \left\lfloor {#2} \right\rfloor
    }
}

\NewDocumentCommand{\modulo}{d<>o}{
    \funcsbr{\mathrm{mod}}<#1>[#2]
}

\NewDocumentCommand{\image}{so}{
    \IfBooleanTF{#1}{
        \funcsbr*{\mathrm{Im}}[#2]
    }{
        \funcsbr{\mathrm{Im}}[#2]
    }
}
\DeclareMathOperator*{\E}{\mathbb{E}}
\NewDocumentCommand{\expect}{sd<>d||m}{
    \IfBooleanTF{#1}{
        \funcmbr*{\E}<#2>|#3|[#4]
    }{
        \funcmbr{\E}<#2>|#3|[#4]
    }
}
\NewDocumentCommand{\prob}{sd<>d||m}{
    \IfBooleanTF{#1}{
        \funclbr*{\Pr}<#2>|#3|[#4]
    }{
        \funclbr{\Pr}<#2>|#3|[#4]
    }
}
\NewDocumentCommand{\distr}{sd<>d()d||}{
    \IfBooleanTF{#1}{
        \ScriptedMathSymbol{\hat{\mathcal{D}}}<#2>(#3)|#4|
    }{
        \ScriptedMathSymbol{\mathcal{D}}<#2>(#3)|#4|
    }
}
\NewDocumentCommand{\gaussian}{d||O{0}O{1}}{
    \ScriptedMathSymbol{\mathcal{N}}|#1|\sbr*{#2,#3}
}

\NewDocumentCommand{\norm}{smd<>d||}{
    \IfBooleanTF{#1}{
        \ScriptedMathSymbol{\left\lVert {#2} \right\rVert}<#3>|#4|
    }{
        \ScriptedMathSymbol{\lVert {#2} \rVert}<#3>|#4|
    }
}
\NewDocumentCommand{\pnorm}{smd<>d||}{
    \IfBooleanTF{#1}{
        \ScriptedMathSymbol{\hat{\lVert} {#2} \hat{\rVert}}<#3>|#4|
    }{
        \ScriptedMathSymbol{\hat{\lVert} {#2} \hat{\rVert}}<#3>|#4|
    }
}
\NewDocumentCommand{\abs}{smd||}{
    \IfBooleanTF{#1}{
        \ScriptedMathSymbol{| {#2} |}|#3|
    }{
        \ScriptedMathSymbol{\left| {#2} \right|}|#3|
    }
}
\NewDocumentCommand{\innerprod}{smmd||}{
    \IfBooleanTF{#1}{
        \ScriptedMathSymbol{\langle #2, #3 \rangle}|#4|
    }{
        \ScriptedMathSymbol{\left\langle #2, #3 \right\rangle}|#4|
    }
}

\NewDocumentCommand{\identity}{d<>d||}{
    \ScriptedMathSymbol{I}<#1>|#2|
}

\NewDocumentCommand{\indicator}{sd<>o}{
    \IfBooleanTF{#1}{
        \funcmbr*{\mathds{1}}<#2>[#3]
    }{
        \funcmbr{\mathds{1}}<#2>[#3]
    }
}

\NewDocumentCommand{\bigO}{sm}{
    \IfBooleanTF{#1}{
        \ScriptedMathSymbol{\tilde{O}}\sbr*{#2}
    }{
        \ScriptedMathSymbol{O}\sbr*{#2}
    }
}
\NewDocumentCommand{\bigM}{sm}{
    \IfBooleanTF{#1}{
        \ScriptedMathSymbol{\tilde{\Omega}}\sbr{#2}
    }{
        \ScriptedMathSymbol{\Omega}\sbr{#2}
    }
}
\NewDocumentCommand{\poly}{so}{
    \IfBooleanTF{#1}{
        \funcsbr*{\mathrm{poly}}[#2]
    }{
        \funcsbr{\mathrm{poly}}[#2]
    }
}
\NewDocumentCommand{\repclass}{md<>d()d||}{
    \ScriptedMathSymbol{\mathcal{#1}}<#2>(#3)|#4|
}

%% file: figures.tex
\usepackage{tikz}
\usepackage{tikz-3dplot-circleofsphere}
\usetikzlibrary{intersections,calc,decorations.pathmorphing,decorations.markings,decorations.pathreplacing}

\definecolor{mypink}{HTML}{D691A4}
\definecolor{darkblue}{HTML}{343F65}
\definecolor{myblue}{HTML}{78B9D2}
\definecolor{myorange}{HTML}{FCAF7C}

\NewDocumentCommand{\drawerrorregion}{sD(){1}}{
    \begin{tikzpicture}[scale=#2]

        \coordinate (O) at (0,0);
        \node at (O) [below right] {O};
    
        \draw[->, thick] (-3,0) -- (3,0);
        \node at (3,0) [right] {$\cvector{e}<1>$};
    
        \draw[dashed, gray, thick, opacity = 0.75] (-0.75,-0.75) -- (2.12,2.12); % lines restricted within the circle
        \node at (2.12,2.12) [above right] {};
    
        \fill[myorange, opacity=0.5] (0,0) -- (3,0) arc[start angle=0, end angle=45, radius=3] -- cycle;
        \fill[myblue, opacity=0.5] (0,0) -- (2.12,2.12) arc[start angle=45, end angle=180, radius=3] -- cycle;
    
        \draw[->, line width=0.5mm] (O) -- (0,3) node[above] {$\cvector{w}$};
    
        \draw[->, line width=0.5mm] (O) -- (-2.12,2.12) node[above left] {$\cvector{v}$};
    
        \draw (0.3,0) -- (0.3,0.3) -- (0,0.3);
    
        \draw (-0.2121,0.2121) -- (-0.424,0) -- (-0.2121,-0.2121);
    
        \draw[-, thick] (0.6,0) arc (0:45:0.6);
        \node at (0.6,0.35) [right] {$\theta(\cvector{v},\cvector{w})$};

        \IfBooleanTF{#1}{
            \def\const{1.5}
            \draw[myorange, thick, opacity = 0.75] ({sqrt(8)},1) -- ({\const*sqrt(8)},\const);
            \node at ({\const*sqrt(8)},\const) [above right] {$I$};
        }{}
    
    \end{tikzpicture}
}

\NewDocumentCommand{\drawclosedcontour}{mmm}{
    \begin{tikzpicture}[scale=#1]

        \coordinate (O) at (0,0);
        \node at (O) [below right] {O};
    
        \draw[->] (-1,0) -- (5,0) node[right] {$t$};
        \draw[->] (0,-3) -- (0,1) node[above] {$i$};

        \draw[myorange, ->, line width=0.5mm] (1,-2) -- (4,-2);
        \node at (4,-2) [below] {$(R, -#2)$};
        \draw[myblue, ->, thick] (4,-2) -- (4,0);
        \node at (4,0) [above] {$(R, 0)$};
        \draw[myblue, ->, thick] (4,0) -- (1,0);
        \node at (1,0) [above] {$(#3, 0)$};
        \draw[myblue, ->, thick] (1,0) -- (1,-2);
        \node at (1,-2) [below] {$(#3, -#2)$};
    
    \end{tikzpicture}
}

%% file: appendix.tex
\section{Omitted Proofs of \cref{sec:positive-reliable-learning-of-homogeneous-halfspaces}}
\label{sec:omitted-proofs-of-positive-reliable-learning}
    \begin{corollary}[\corollaryref{cor:hardness-of-positive-reliable-learning-over-homogeneous-halfspace}]
    \label{cor:hardness-of-positive-reliable-learning-over-homogeneous-halfspace-appendix}
        Given \cref{asp:sub-exponential-assumption-of-lwe}, for any constants $\alpha\in(0, 1), \gamma > 1/2$, and any $c/\sqrt{d\log d}\leq \epsilon\leq 1/\log^\gamma d$ where $c$ is a sufficiently small constant, there is no $(1, \epsilon)$-approximation algorithm that solves \cref{prob:positive-reliable-learning} and runs in time $\cscalar{d}|\bigO{1/\sbr{\cscalar{\epsilon}\sqrt{\log d}}|2\alpha|}|$.
    \end{corollary}
    \begin{proof}
       Suppose there exists a $(1, \epsilon)$-approximation algorithm that solves \cref{prob:positive-reliable-learning}. 
       
       Then, for the alternative hypothesis as described in Theorem \ref{thm:indistinguishability-of-one-sided-homogeneous-halfspaces}, we have that
       \begin{equation*}
           \prob<(\rvector{x}, \rscalar{y})\sim\distr>{\rscalar{y} = 1\cond\innerprod{\cvector{s}}{\rvector{x}} \geq 0}\geq \frac{1}{2} + \frac{C}{\sqrt{\eta\log d}}
       \end{equation*}
       where $C > 0$ is a universal constant. Running the algorithm mentioned above may return us a $\cvector{s} '$ such that
       \begin{equation*}
           \prob<(\rvector{x}, \rscalar{y})\sim\distr>{\rscalar{y} = 1\cond\innerprod{\cvector{s} '}{\rvector{x}} \geq 0} \geq \frac{1}{2} + \frac{C}{\sqrt{\eta\log d}} - \epsilon.
       \end{equation*}
       By the \emph{Hoeffding Bound}, we may estimate $\prob<\distr>{\rscalar{y} = 1\cond\innerprod{\cvector{s} '}{\rvector{x}} \geq 0}$ up to $1/2 + C/\sqrt{\eta\log d} - \epsilon - \varepsilon_1$ by drawing $O(1/\cscalar{\varepsilon}<1>|2|)$ examples from the distribution constructed in the alternative hypothesis. For the null hypothesis, we can verify there is no $\cvector{u}\in\real|d|$ such that $\prob<\distr>{\rscalar{y} = 1\cond\innerprod{\cvector{u}}{\rvector{x}} \geq 0} \geq 1/2 + \cscalar{\varepsilon}<2>$ given $O(1/\cscalar{\varepsilon}<2>|2|)$ examples from the distribution in the null hypothesis case.
       
       Observe that, if $\epsilon = c/\sqrt{\eta\log d}$ for some sufficiently small constant $c$, we only need $\varepsilon_2=\varepsilon_1 = \bigO{1/\sqrt{\eta\log d}}$ to ensure $C/\sqrt{\eta\log d} - \epsilon - \varepsilon_1 > \varepsilon_2$, and to solve the decision problem in Theorem \ref{thm:indistinguishability-of-one-sided-homogeneous-halfspaces} within time $\cscalar{d}|\bigO{\cscalar{\eta}|\alpha|}|$ for any $\alpha\in(0,1)$ by running the above algorithm. Given $\epsilon = c/\sqrt{\eta\log d}$, we can rewrite $\cscalar{d}|\bigO{\cscalar{\eta}|\alpha|}| = \cscalar{d}|\bigO{1/\sbr{\cscalar{\epsilon}\sqrt{\log d}}|2\alpha|}|$. However, Theorem \ref{thm:indistinguishability-of-one-sided-homogeneous-halfspaces} tells that the above case is impossible for any $c/\sqrt{d\log d}\leq \epsilon\leq 1/\log^\gamma d$.
    \end{proof}
    
\section{Omitted Proofs of \cref{sec:agnostically-learning-homogeneous-halfspaces}}
\label{sec:omitted-proofs-of-agnostic-learning}
    \begin{theorem}[\cref{thm:indistinguishability-of-agnostic-learning}]
    \label{thm:indistinguishability-of-agnostic-learning-appendix}
        Under \cref{asp:sub-exponential-assumption-of-lwe}, for any $d\in\naturals$, any constants $\alpha\in(0, 1), \gamma\in\real<+>$, and any $\log^{\gamma} d\leq \cscalar{\eta} \leq \cscalar{c}d$ where $\cscalar{c}$ is a sufficiently small constant, there is no algorithm that runs in time $\cscalar{d}|\bigO{\cscalar{\eta}|\alpha|}|$ and distinguishes between the following two cases of a joint distribution $\distr$ of $(\rvector{x}, \rscalar{y})$ supported on $\real|d|\times\binarydomain$ such that the marginal $\distr<\rvector{x}> = \gaussian|d|$, with $\cscalar{d}|-\bigO{\cscalar{\eta}|\alpha|} |$ advantage:
        \begin{itemize}
            \item[\textbf{Alternative Hypothesis}:]$\exists\cvector{v}\in \sphere|d-1|$ such that $\prob<(\rvector{x}, \rscalar{y})\sim\distr>{\rscalar{y} = \funcsbr{\sgn}[\innerprod{\cvector{v}}{\rvector{x}}]} = \frac{1}{2} + \bigM{\frac{1}{\sqrt{\eta\log d}}}$.
            \item[\textbf{Null Hypothesis}:]$\forall\cvector{u}\in \sphere|d-1|$, it holds that $\prob<(\rvector{x}, \rscalar{y})\sim\distr>{\rscalar{y} = \funcsbr{\sgn}[\innerprod{\cvector{u}}{\rvector{x}}]} = \frac{1}{2}$.
        \end{itemize}
    \end{theorem}
    \begin{proof}
        Similar to the proof of \cref{thm:indistinguishability-of-one-sided-homogeneous-halfspaces}, we give an efficient method taking as input samples $(\rvector{x}, \rscalar{y})\sim\distr '$ either from the alternative or the null hypothesis of $\text{LWE}(\cscalar{d}|\bigO{\cscalar{\eta}|\alpha|}|, \gaussian|d|, \sphere|d-1|,$ $\gaussian[0][\cscalar{\sigma}|2|], \modulo<T>)$ as in \lemmaref{lma:sub-exponential-hardness-of-clwe}, and generating samples from another distribution $\distr$ with the following properties: if $\distr '$ is from the alternative (resp. null) hypothesis of the LWE problem, then the resulting distribution $\distr$ will satisfy the alternative (resp. null) hypothesis requirement of the theorem statement.
    
        The reduction is the same as that in the proof of \cref{thm:indistinguishability-of-one-sided-homogeneous-halfspaces}. For convenience, we restate the reduction process as follows: for $(\rvector{x}, \rscalar{y '})$ sampled from a instance $\distr'$ of the problem LWE$(\cscalar{d}|\bigO{\cscalar{\eta}|\alpha|}|,\gaussian|d|, \sphere|d-1|, \gaussian[0][\sigma^2],\modulo<T>)$ as in \lemmaref{lma:sub-exponential-hardness-of-clwe}, we simply output $(\rvector{x}, \rscalar{y})\sim\distr$, such that
        \begin{equation*}
            \rscalar{y}=
            \begin{cases}
                +1,\quad \text{if }\rscalar{y'}\leq T/2
                \\
                -1,\quad \text{otherwise}.
            \end{cases}
        \end{equation*}
        Our proof idea is to first reduce the ``two-sided'' classification agreement, $\prob{\rscalar{y} = \funcsbr{\sgn}[\innerprod{\cvector{u}}{\rvector{x}}]}$, to the ``one-sided'' version, $\prob{\rscalar{y} = 1\cap\innerprod{\cvector{u}}{\rvector{x}} \geq 0}$, using the ``symmetry'' of Gaussian distributions, and the rest of the proof follows that of \cref{thm:indistinguishability-of-one-sided-homogeneous-halfspaces}.
        
        By the \emph{law of total probability}, it holds, for any $\cvector{u}\in\sphere|d-1|$, that
        \begin{equation}
            \prob<(\rvector{x}, \rscalar{y})\sim\distr>{\rscalar{y} = \funcsbr{\sgn}[\innerprod{\cvector{u}}{\rvector{x}}]} = \prob<(\rvector{x}, \rscalar{y})\sim\distr>{\rscalar{y} = 1\cap\innerprod{\cvector{u}}{\rvector{x}} \geq 0} + \prob<(\rvector{x}, \rscalar{y})\sim\distr>{\rscalar{y} = -1\cap\innerprod{\cvector{u}}{\rvector{x}} < 0}\label{eq:decomposition-of-classification-correctness}
        \end{equation}
        \textbf{For the alternative hypothesis case}, let $\distr<\rvector{x}, \rscalar{z}>$ be a product distribution of $\distr<\rvector{x}>$ and $\gaussian[0][\sigma^2]$. We have that
        \begin{align}
            \prob<(\rvector{x}, \rscalar{y})\sim\distr>{\rscalar{y} = -1\cap\innerprod{\cvector{s}}{\rvector{x}} < 0} =&  \prob<(\rvector{x}, \rscalar{z})\sim\distr<\rvector{x}, \rscalar{z}>>{\modulo<T>[\innerprod{\cvector{s}}{\rvector{x}} + \rscalar{z}]\geq T/2\cap \innerprod{\cvector{s}}{\rvector{x}} < 0}
            \notag
            \\
            \ceq[i]& \prob<(\rvector{x}, \rscalar{z})\sim\distr<\rvector{x}, \rscalar{z}>>{\modulo<T>[-\innerprod{\cvector{s}}{\rvector{x}} - \rscalar{z}] < T/2\cap \innerprod{\cvector{s}}{-\rvector{x}} \geq 0}
            \notag
            \\
            \ceq[ii]& \prob<(-\rvector{x}, -\rscalar{z})\sim\distr<-\rvector{x}, -\rscalar{z}>>{\modulo<T>[\innerprod{\cvector{s}}{-\rvector{x}} + (- \rscalar{z})] < T/2\cap \innerprod{\cvector{s}}{-\rvector{x}} \geq 0}
            \notag
            \\
            \ceq[iii]& \prob<(\rvector{x}, \rscalar{z})\sim\distr<\rvector{x}, \rscalar{z}>>{\modulo<T>[\innerprod{\cvector{s}}{\rvector{x}} + \rscalar{z}] < T/2\cap \innerprod{\cvector{s}}{\rvector{x}} \geq 0}
            \notag
            \\
            =& \prob<(\rvector{x}, \rscalar{y})\sim\distr>{\rscalar{y} = 1\cap\innerprod{\cvector{s}}{\rvector{x}} \geq 0}\label{eq:equivalence-between-two-one-sided-classification-correctness}
        \end{align}
        where Equation (i) holds because $\lbr{t\in\real\cond \modulo<T>[t]\geq T/2}\equiv \lbr{t\in\real\cond\modulo<T>[-t] < T/2}$; Equation (ii) holds because the independence between $\rvector{x}$ and $\rscalar{z}$ implies independence between $-\rvector{x}$ and $-\rscalar{z}$; and Equation (iii) holds due to the \emph{rotational invariance} of Gaussian distributions. 
        
        Now, plugging Equation \eqref{eq:equivalence-between-two-one-sided-classification-correctness} back in to Equation \eqref{eq:decomposition-of-classification-correctness} and comparing to Equation \eqref{eq:lower-bound-of-the-joint-event-for-one-sided-homogeneous-halfspaces} gives
        \begin{align}
             \prob<(\rvector{x}, \rscalar{y})\sim\distr>{\rscalar{y} = \funcsbr{\sgn}[\innerprod{\cvector{v}}{\rvector{x}}]} =& 2\prob<(\rvector{x}, \rscalar{y})\sim\distr>{\rscalar{y} = 1\cap\innerprod{\cvector{s}}{\rvector{x}} \geq 0}
             \notag
             \\
             \geq& \frac{1}{2} + \frac{T\cscalar{e}|- 2\pi^2\sigma^2/T^2|}{2\pi^2}.\label{eq:lower-bound-of-the-joint-event-for-agnostic-learning}
        \end{align}

        Recall that \lemmaref{lma:sub-exponential-hardness-of-clwe} states that the LWE problem is hard for any fixed constant $\kappa\in \naturals$ and $\sigma\geq \cscalar{\eta}|-\kappa|$. Given the constant $\gamma\in \real<+>$ in this theorem, we can take $\kappa = \ceil{1/2\gamma + 1/2}$, which is a fixed constant. Then, by \lemmaref{lma:sub-exponential-hardness-of-clwe}, the LWE problem is hard for $\sigma = \cscalar{\eta}|-\kappa|\leq 1/(\sqrt{\cscalar{\eta}\log d}) = O(T)$. Thus, taking $\sigma = \cscalar{\eta}|-\kappa|$ back to Equation \eqref{eq:lower-bound-of-the-joint-event-for-agnostic-learning} with $T = 1/C'\sqrt{\eta\log d}$ (cf. \lemmaref{lma:sub-exponential-hardness-of-clwe}) gives
        \begin{equation*}
            \prob<(\rvector{x}, \rscalar{y})\sim\distr>{\rscalar{y} = \funcsbr{\sgn}[\innerprod{\cvector{v}}{\rvector{x}}]} = \frac{1}{2} + \bigM{\frac{1}{\sqrt{\eta\log d}}}.
        \end{equation*}
        
        \textbf{For the null hypothesis}, because $\rscalar{y}$ is generated independently at random, half of the data points will be mislabeled regardless of which halfspace we choose. Therefore, it holds that $\prob<(\rvector{x}, \rscalar{y})\sim\distr>{\rscalar{y} = \funcsbr{\sgn}[\innerprod{\cvector{u}}{\rvector{x}}]} = 1/2$ for all $\cvector{u}\in\sphere|d-1|$.
    
       Finally, we verify the time complexity and the distinguishing advantage for agnostically learning homogeneous halfspaces. By \lemmaref{lma:sub-exponential-hardness-of-clwe}, the problem LWE$(\cscalar{d}|\bigO{\cscalar{\eta}|\alpha|}|, \gaussian|d|, \sphere|d-1|, \gaussian[0][\sigma^2], \modulo<T>)$ cannot be solved in $\cscalar{d}|\bigO{\cscalar{\eta}|\alpha|}|$ time with $\cscalar{d}|-\bigO{\cscalar{\eta}|\alpha|}|$ advantage with any choice of $\alpha\in(0,1)$, $\sigma\geq \cscalar{\eta}|-\kappa|$ (where $\kappa\in \naturals$ is a constant) and $T = 1/C'\sqrt{\cscalar{\eta}\log d}$ (where $C' > 0$ is a sufficiently large constant), under \cref{asp:sub-exponential-assumption-of-lwe}. Therefore, under the same assumption, no algorithm can solve the decision problem as in the statement in $\cscalar{d}|\bigO{\cscalar{\eta}|\alpha|}|$ time with $\cscalar{d}|-\bigO{\cscalar{\eta}|\alpha|}|$ advantage.
    \end{proof}

    \begin{corollary}[\corollaryref{cor:hardness-of-learning-homogeneous-halfspace-under-gaussian}]
    \label{cor:hardness-of-learning-homogeneous-halfspace-under-gaussian-appendix}
        Given \cref{asp:sub-exponential-assumption-of-lwe}, for any constants $\alpha\in(0, 1), \gamma > 1/2$, and any $c/\sqrt{d\log d}\leq \epsilon\leq 1/\log^\gamma d$ where $c$ is a sufficiently small constant, there is no $(1, \epsilon)$-approximation algorithm that solves \cref{prob:agnostic-learning} and runs in time $\cscalar{d}|\bigO{1/\sbr{\cscalar{\epsilon}\sqrt{\log d}}|2\alpha|}|$.
    \end{corollary}
    \begin{proof}
       Suppose there exists a $(1, \epsilon)$-approximation algorithm that solves \cref{prob:agnostic-learning}. 

       For the alternative hypothesis as described in Theorem \ref{thm:indistinguishability-of-agnostic-learning-appendix}, we have that
       \begin{align*}
           \prob<(\rvector{x}, \rscalar{y})\sim\distr>{\rscalar{y} \neq \funcsbr{\sgn}[\innerprod{\cvector{s}}{\rvector{x}}]} =& 1 - \prob<(\rvector{x}, \rscalar{y})\sim\distr>{\rscalar{y} = \funcsbr{\sgn}[\innerprod{\cvector{s}}{\rvector{x}}]}
           \\
           \leq& \frac{1}{2} - \frac{C}{\sqrt{\eta\log d}}
       \end{align*}
       where $C > 0$ is a universal constant. Running the algorithm mentioned above may return us a $\cvector{s} '$ such that
       \begin{equation*}
           \prob<(\rvector{x}, \rscalar{y})\sim\distr>{\rscalar{y} \neq \funcsbr{\sgn}[\innerprod{\cvector{s} '}{\rvector{x}}]}\leq \frac{1}{2} - \frac{C}{\sqrt{\eta\log d}} + \epsilon.
       \end{equation*}
       Again, by the \emph{Hoeffding Bound}, we can estimate $\prob<\distr>{\rscalar{y} \neq \funcsbr{\sgn}[\innerprod{\cvector{s} '}{\rvector{x}}]}$ up to $1/2 - 1/\sqrt{\eta\log d} + \epsilon + \varepsilon_1$ by drawing $O(1/\cscalar{\varepsilon}<1>|2|)$ examples from the distribution constructed in the alternative hypothesis. For the null hypothesis, we can verify there is no $\cvector{u}\in\real|d|$ such that $\prob<\distr>{\rscalar{y} \neq \funcsbr{\sgn}[\innerprod{\cvector{u}}{\rvector{x}}]} \leq 1/2 - \cscalar{\varepsilon}<2>$ given $O(1/\cscalar{\varepsilon}<2>|2|)$ examples from the distribution in the null hypothesis case.
       
       Observe that, if $\epsilon = c/\sqrt{\eta\log d}$ for some sufficiently small constant $c$, we only need $\varepsilon_2=\varepsilon_1 = \bigO{1/\sqrt{\eta\log d}}$ to ensure $C/\sqrt{\eta\log d} - \epsilon - \varepsilon_1 > \varepsilon_2$, and to solve the decision problem in Theorem \ref{thm:indistinguishability-of-agnostic-learning-appendix} within time $\cscalar{d}|\bigO{\cscalar{\eta}|\alpha|}|$ for any $\alpha\in(0,1)$ by running the above algorithm. Given $\epsilon = c/\sqrt{\eta\log d}$, we can rewrite $\cscalar{d}|\bigO{\cscalar{\eta}|\alpha|}| = \cscalar{d}|\bigO{1/\sbr{\cscalar{\epsilon}\sqrt{\log d}}|2\alpha|}|$. However, Theorem \ref{thm:indistinguishability-of-agnostic-learning-appendix} tells that the above case is impossible for any $c/\sqrt{d\log d}\leq \epsilon\leq 1/\log^\gamma d$.
    \end{proof}

\section{Omitted Proofs of \cref{sec:fairness-auditing-over-homogeneous-halfspace-subgroups}}
\label{sec:omitted-proofs-of-fairness-auditing}
    \begin{theorem}[\cref{thm:indistinguishability-of-fairness-auditing}]
    \label{thm:indistinguishability-of-fairness-auditing-appendix}
        Under \cref{asp:sub-exponential-assumption-of-lwe}, for any $d\in\naturals$, any constants $\alpha\in(0, 1), \gamma\in\real<+>$, and any $\log^{\gamma} d\leq \cscalar{\eta} \leq \cscalar{c}d$ where $\cscalar{c}$ is a sufficiently small constant, there is no algorithm that runs in time $\cscalar{d}|\bigO{\cscalar{\eta}|\alpha|}|$ and distinguishes between the following two cases of a joint distribution $\distr$ of $(\rvector{x}, \rscalar{y})$ supported on $\real|d|\times\binarydomain$ such that the marginal $\distr<\rvector{x}> = \gaussian|d|$, with $\cscalar{d}|-\bigO{\cscalar{\eta}|\alpha|} |$ advantage:
        \begin{itemize}
            \item[\textbf{Alternative Hypothesis}:]$\exists\cvector{v}\in \sphere|d-1|$ such that $\funcsbr{u}<\distr>[\cvector{v}] = \bigM{1/\sqrt{\eta\log d}}$.
            \item[\textbf{Null Hypothesis}:]$\forall\cvector{u}\in \sphere|d-1|$, it holds that $\funcsbr{u}<\distr>[\cvector{u}] = 0$.
        \end{itemize}
    \end{theorem}
    \begin{proof}
        Similar to the proof of \cref{thm:indistinguishability-of-one-sided-homogeneous-halfspaces}, we give an efficient method taking as input samples $(\rvector{x}, \rscalar{y})\sim\distr '$ either from the alternative or the null hypothesis of $\text{LWE}(\cscalar{d}|\bigO{\cscalar{\eta}|\alpha|}|, \gaussian|d|, \sphere|d-1|,$ $\gaussian[0][\cscalar{\sigma}|2|], \modulo<T>)$ as in \lemmaref{lma:sub-exponential-hardness-of-clwe}, and generating samples from another distribution $\distr$ with the following properties: if $\distr '$ is from the alternative (resp. null) hypothesis of the LWE problem, then the resulting distribution $\distr$ will satisfy the alternative (resp. null) hypothesis requirement of the theorem statement.
    
        Again, the reduction process is essentially the same as that in the proof of \cref{thm:indistinguishability-of-one-sided-homogeneous-halfspaces}: for $(\rvector{x}, \rscalar{y '})$ sampled from a instance $\distr'$ of the problem LWE$(\cscalar{d}|\bigO{\cscalar{\eta}|\alpha|}|,\gaussian|d|, \sphere|d-1|, \gaussian[0][\sigma^2],\modulo<T>)$ as in \lemmaref{lma:sub-exponential-hardness-of-clwe}, we simply output $(\rvector{x}, \rscalar{y})\sim\distr$, such that
        \begin{equation*}
            \rscalar{y}=
            \begin{cases}
                +1,\quad \text{if }\rscalar{y'}\leq T/2
                \\
                -1,\quad \text{otherwise}.
            \end{cases}
        \end{equation*}
    
        \textbf{For the alternative hypothesis case}, let $\distr<\rvector{x}, \rscalar{z}>$ be a product distribution of $\distr<\rvector{x}>$ and $\gaussian[0][\sigma^2]$, similar to the proof of \cref{thm:indistinguishability-of-agnostic-learning-appendix}, we have that
        \begin{align}
            \prob<\rscalar{y}\sim\distr<\rscalar{y}>>{\rscalar{y} = 1} =&  \prob<(\rvector{x}, \rscalar{z})\sim\distr<\rvector{x}, \rscalar{z}>>{\modulo<T>[\innerprod{\cvector{s}}{\rvector{x}} + \rscalar{z}]< T/2}
            \notag
            \\
            \ceq[i]& \prob<(\rvector{x}, \rscalar{z})\sim\distr<\rvector{x}, \rscalar{z}>>{\modulo<T>[-\innerprod{\cvector{s}}{\rvector{x}} - \rscalar{z}] < T/2}
            \notag
            \\
            \ceq[ii]& \prob<(\rvector{x}, \rscalar{z})\sim\distr<\rvector{x}, \rscalar{z}>>{\modulo<T>[\innerprod{\cvector{s}}{\rvector{x}} + \rscalar{z}] \geq T/2}
            \notag
            \\
            =& \prob<\rscalar{y}\sim\distr<\rscalar{y}>>{\rscalar{y} = -1}\label{eq:parity-symmetry-of-the-lwe-labels}
        \end{align}
        where Equation (i) holds due to the \emph{rotational invariance} of Gaussian distributions and Equation (ii) holds because $\lbr{t\in\real\cond \modulo<T>[t]\geq T/2}\equiv \lbr{t\in\real\cond\modulo<T>[-t] < T/2}$.
        
        Observe now that Equation \eqref{eq:parity-symmetry-of-the-lwe-labels} implies $ \prob<\rscalar{y}\sim\distr<\rscalar{y}>>{\rscalar{y} = 1} = 1/2$, which, along with Equation \eqref{eq:lower-bound-of-the-joint-event-for-one-sided-homogeneous-halfspaces}, gives
        \begin{equation}
             \prob<(\rvector{x}, \rscalar{y})\sim\distr>{\rscalar{y} = 1\cap\innerprod{\cvector{v}}{\rvector{x}}\geq 0} - \prob<\rvector{x}\sim\distr<\rvector{x}>>{\innerprod{\cvector{v}}{\rvector{x}}\geq 0}\prob<\rscalar{y}\sim\distr<\rscalar{y}>>{\rscalar{y} = 1} \geq \frac{T\cscalar{e}|- 2\pi^2\sigma^2/T^2|}{4\pi^2}. \label{eq:lower-bound-of-the-joint-event-for-fairness-auditing}
        \end{equation}

        Recall that \lemmaref{lma:sub-exponential-hardness-of-clwe} states that the LWE problem is hard for any fixed constant $\kappa\in \naturals$ and $\sigma\geq \cscalar{\eta}|-\kappa|$. Given the constant $\gamma\in \real<+>$ in this theorem, we can take $\kappa = \ceil{1/2\gamma + 1/2}$, which is a fixed constant. Then, by \lemmaref{lma:sub-exponential-hardness-of-clwe}, the LWE problem is hard for $\sigma = \cscalar{\eta}|-\kappa|\leq 1/(\sqrt{\cscalar{\eta}\log d}) = O(T)$. Thus, taking $\sigma = \cscalar{\eta}|-\kappa|$ back to equation \eqref{eq:lower-bound-of-the-joint-event-for-fairness-auditing} with $T = 1/C'\sqrt{\eta\log d}$ (cf. \lemmaref{lma:sub-exponential-hardness-of-clwe}) gives
        \begin{equation*}
            \prob<(\rvector{x}, \rscalar{y})\sim\distr>{\rscalar{y} = 1\cap\innerprod{\cvector{v}}{\rvector{x}}\geq 0} - \prob<\rvector{x}\sim\distr<\rvector{x}>>{\innerprod{\cvector{v}}{\rvector{x}}\geq 0}\prob<\rscalar{y}\sim\distr<\rscalar{y}>>{\rscalar{y} = 1} = \bigM{\frac{1}{\sqrt{\eta\log d}}}.
        \end{equation*}
        
        \textbf{For the null hypothesis}, because $\rscalar{y}$ is generated independently at random it holds, for all $\cvector{u}\in\sphere|d-1|$, that 
        \begin{equation*}
            \prob<(\rvector{x}, \rscalar{y})\sim\distr>{\rscalar{y} = 1\cap\innerprod{\cvector{u}}{\rvector{x}}\geq 0} = \prob<\rvector{x}\sim\distr<\rvector{x}>>{\innerprod{\cvector{u}}{\rvector{x}}\geq 0}\prob<\rscalar{y}\sim\distr<\rscalar{y}>>{\rscalar{y} = 1}
        \end{equation*}
        which in turn implies
        \begin{equation*}
            \prob<(\rvector{x}, \rscalar{y})\sim\distr>{\rscalar{y} = 1\cap\innerprod{\cvector{u}}{\rvector{x}}\geq 0} - \prob<\rvector{x}\sim\distr<\rvector{x}>>{\innerprod{\cvector{u}}{\rvector{x}}\geq 0}\prob<\rscalar{y}\sim\distr<\rscalar{y}>>{\rscalar{y} = 1} = 0,\ \forall \cvector{u}\in\sphere|d-1|.
        \end{equation*}
    
        Finally, we verify the time lower bound and the distinguishing advantage for auditing homogeneous halfspace subgroups. By \lemmaref{lma:sub-exponential-hardness-of-clwe}, the problem LWE$(\cscalar{d}|\bigO{\cscalar{\eta}|\alpha|}|, \gaussian|d|, \sphere|d-1|, \gaussian[0][\sigma^2], \modulo<T>)$ cannot be solved in $\cscalar{d}|\bigO{\cscalar{\eta}|\alpha|}|$ time with $\cscalar{d}|-\bigO{\cscalar{\eta}|\alpha|}|$ advantage with any choice of $\alpha\in(0,1)$, $\sigma\geq \cscalar{\eta}|-\kappa|$ (where $\kappa\in \naturals$ is a constant) and $T = 1/C'\sqrt{\cscalar{\eta}\log d}$ (where $C' > 0$ is a sufficiently large constant), under \cref{asp:sub-exponential-assumption-of-lwe}. Therefore, under the same assumption, no algorithm can solve the decision problem as in the statement in $\cscalar{d}|\bigO{\cscalar{\eta}|\alpha|}|$ time with $\cscalar{d}|-\bigO{\cscalar{\eta}|\alpha|}|$ advantage.
    \end{proof}
    
    \begin{corollary}[\corollaryref{cor:hardness-of-auditing-halfspace-under-gaussian-in-multiplicative-form}]
    \label{cor:hardness-of-auditing-halfspace-under-gaussian-in-multiplicative-form-appendix}
        Given \cref{asp:sub-exponential-assumption-of-lwe}, there is no $1/\poly[d]$-approximation algorithm that solves \cref{prob:fairness-auditing} and runs in time $\poly[d]$.
    \end{corollary}
    \begin{proof}
        Suppose there exists a $\delta$-approximation algorithm that solves \cref{prob:fairness-auditing}. 
        
        For the alternative hypothesis as described in Theorem \ref{thm:indistinguishability-of-fairness-auditing-appendix}, we have that $\funcsbr{u}<\distr>[\cvector{s}]\geq C/\sqrt{\eta\log d}$ for some universal constant $C > 0$. Running the algorithm mentioned above may return us a $\cvector{s} '$ such that $ \funcsbr{u}<\distr>[\cvector{s} '] \geq C\delta/\sqrt{\eta\log d}$. By the \emph{Hoeffding Bound}, we can estimate $\funcsbr{u}<\distr>[\cvector{s} ']$ up to $C\delta/\sqrt{\eta\log d} - \varepsilon_1$ by drawing $O(1/\cscalar{\varepsilon}<1>|2|)$ examples from the distribution constructed in the alternative hypothesis. For the null hypothesis case, we can verify there is no $\cvector{u}\in\real|d|$ such that $\funcsbr{u}<\distr>[\cvector{u}] \geq \cscalar{\varepsilon}<2>$ given $O(1/\cscalar{\varepsilon}<2>|2|)$ examples from the distribution in the null hypothesis case.

        Notice that, for any $\delta = 1/\poly[d]$, we only need $\cscalar{\varepsilon}<2>=\cscalar{\varepsilon}<1> = 1/\poly[d]$ to ensure $C\delta/\sqrt{\eta\log d} - \cscalar{\varepsilon}<1> > \cscalar{\varepsilon}<2>$, and to solve the decision problem in Theorem \ref{thm:indistinguishability-of-fairness-auditing-appendix} within time $\poly[d]$ by running the above algorithm. However, this implies that our auditing algorithm can distinguish between the two cases in Theorem \ref{thm:indistinguishability-of-fairness-auditing-appendix} in $\poly[d]$ time, which contradicts the hardness assumption.
    \end{proof}

    \begin{corollary}[\corollaryref{cor:hardness-of-auditing-halfspace-under-gaussian-in-additive-form}]
    \label{cor:hardness-of-auditing-halfspace-under-gaussian-in-additive-form-appendix}
        Given \cref{asp:sub-exponential-assumption-of-lwe}, for any constants $\alpha\in(0, 1), \gamma > 1/2$, and any $c/\sqrt{d\log d}\leq \epsilon\leq 1/\log^\gamma d$ where $c$ is a sufficiently small constant, there is no $(1, \epsilon)$-approximation algorithm that solves \cref{prob:fairness-auditing} and runs in time $\cscalar{d}|\bigO{1/\sbr{\cscalar{\epsilon}\sqrt{\log d}}|2\alpha|}|$.
    \end{corollary}
    \begin{proof}
       Suppose there exists a $(1, \epsilon)$-approximation algorithm that solves \cref{prob:fairness-auditing}. 

       For the alternative hypothesis as described in Theorem \ref{thm:indistinguishability-of-fairness-auditing-appendix}, we have that $\funcsbr{u}<\distr>[\cvector{s}]\geq C/\sqrt{\eta\log d}$, where $C > 0$ is a universal constant. Running the algorithm mentioned above may return us a $\cvector{s} '$ such that $ \funcsbr{u}<\distr>[\cvector{s} '] \geq C/\sqrt{\eta\log d} - \epsilon$. Again, by the \emph{Hoeffding Bound}, we can estimate $\funcsbr{u}<\distr>[\cvector{s} ']$ up to $C/\sqrt{\eta\log d} - \epsilon - \varepsilon_1$ by drawing $O(1/\cscalar{\varepsilon}<1>|2|)$ examples from the distribution constructed in the alternative hypothesis. For null hypothesis, we can verify there is no $\cvector{u}\in\real|d|$ such that $\funcsbr{u}<\distr>[\cvector{u}] \geq \cscalar{\varepsilon}<2>$ given $O(1/\cscalar{\varepsilon}<2>|2|)$ examples from the distribution in the null hypothesis case.
       
       Observe that, if $\epsilon = c/\sqrt{\eta\log d}$ for some sufficiently small constant $c$, we only need $\varepsilon_2=\varepsilon_1 = \bigO{1/\sqrt{\eta\log d}}$ to ensure $C/\sqrt{\eta\log d} - \epsilon - \varepsilon_1 > \varepsilon_2$, and to solve the decision problem in Theorem \ref{thm:indistinguishability-of-fairness-auditing-appendix} within time $\cscalar{d}|\bigO{\cscalar{\eta}|\alpha|}|$ by running the above algorithm. Meanwhile, given $\epsilon = c/\sqrt{\eta\log d}$, we can rewrite $\cscalar{d}|\bigO{\cscalar{\eta}|\alpha|}| = \cscalar{d}|\bigO{1/\sbr{\cscalar{\epsilon}\sqrt{\log d}}|2\alpha|}|$. However, Theorem \ref{thm:indistinguishability-of-fairness-auditing-appendix} tells that the above case is impossible for any $C/\sqrt{d\log d}\leq \epsilon\leq c/\log^\gamma d$.
    \end{proof}
    
\section{Sub-Gaussian Fourier Transform}
    
    \begin{fact}[Fourier Series Of Square Wave]
    \label{fac:fourier-series-of-square-wave}
        For any square wave function $\funcsbr{f}[t] = \funcsbr{\sgn}[\funcsbr{\sin}[t]]$, it holds that 
        \begin{equation*}
            \funcsbr{f}[t] = \frac{4}{\pi}\sum_{k=1}^{+\infty}\frac{\funcsbr{\sin}[\sbr{2k-1}t]}{2k - 1}.
        \end{equation*}
    \end{fact}

    \begin{fact}[Fourier Transform of $\funcsbr{\phi}<\mu, \sigma>$ \citep{bryc1995normal}]
    \label{fac:fourier-transform-of-gaussian}
        For any $\omega\in\real$, it holds that $\expect<\rscalar{z}\sim\gaussian[\mu][\sigma^2]>{\cscalar{e}|i\omega\rscalar{z}|} = \cscalar{e}|i\mu\omega - \sigma^2\omega^2/2|$.
    \end{fact}

    \begin{lemma}[Cauchy–Goursat Theorem]
    \label{lma:cauchy-goursat-theorem}
        Let $U\subseteq \complex$ be a simply connected open set and $f:U\rightarrow\complex$ be a analytic(holomorphic) function on $U$. Then, for every closed contour $C\subseteq U$ on the complex plane, it holds that $\oint_C \funcsbr{f}[z]dz = 0$.
    \end{lemma}
    
    \begin{lemma}[Sub-Gaussian FT]
    \label{lma:sub-gaussian-fourier-transform}
        For any $\omega\in \real$, it holds that $\expect<\rscalar{z}\sim\gaussian[\mu][\sigma^2]>{\funcsbr{\sin}[\omega\rscalar{z}]} = \funcsbr{\sin}[\mu\omega]\cscalar{e}|-\sigma^2\omega^2/2|$.
    \end{lemma}
    \begin{proof}
        Notice that, by \emph{Euler's formula}, we have that
        \begin{align*}
             \expect<\rscalar{z}\sim\gaussian[\mu][\sigma^2]>{\funcsbr{\sin}[\omega\rscalar{z}]} =& \expect<\rscalar{z}\sim\gaussian[\mu][\sigma^2]>{\image[\cscalar{e}|i\omega\rscalar{z}|]}
             \\
             =& \image*[\expect<\rscalar{z}\sim\gaussian[\mu][\sigma^2]>{\cscalar{e}|i\omega\rscalar{z}|}]
             \\
             \ceq[i]& \image[\cscalar{e}|i\mu\omega|\cscalar{e}|-\sigma^2\omega^2/2|]
             \\
             =& \funcsbr{\sin}[\mu\omega]\cscalar{e}|-\sigma^2\omega^2/2|
        \end{align*}
        where Equation (i) holds because of \cref{fac:fourier-transform-of-gaussian}. The last equation is obtained by another application of Euler's formula.
    \end{proof}

    \begin{lemma}[Centering of Partial Sub-Gaussian FT]
    \label{lma:centering-of-partial-sub-gaussian-ft}
        For any $a > 0$, $\alpha\geq 0$, and $\omega\in\real$, it holds that
        \begin{equation*}
            \int_{\alpha}^{+\infty} \cscalar{e}|-at^2|\funcsbr{\sin}[\omega t] dt = \frac{\cscalar{e}|-a\alpha^2 - \omega^2/4a|}{\sqrt{a}} \int_{0}^{\omega/2\sqrt{a}} \cscalar{e}|t^2|\funcsbr{\cos}[2\alpha\sqrt{a} t] dt.
        \end{equation*}
    \end{lemma}
    \begin{proof}
        By \emph{Euler's formula}, we have that
        \begin{align}
            \int_{\alpha}^{+\infty} \cscalar{e}|-at^2|\funcsbr{\sin}[\omega t] dt =& \image*[\int_{\alpha}^{+\infty} \cscalar{e}|-at^2 + i\omega t| dt]
            \notag
            \\
            =& \cscalar{e}|-\omega^2/4a|\image*[\int_{\alpha}^{+\infty} \cscalar{e}|-a(t - i\omega/2a)^2| dt]
            \notag
            \\
            =& \cscalar{e}|-\omega^2/4a|\image*[\int_{\alpha - i\omega/2a}^{+\infty - i\omega/2a} \cscalar{e}|-az^2| dz]
            \label{eq:applying-euler-formula-and-change-variable-to-partial-fourier-transform}
        \end{align}
        where the last equation is obtained by changing the variable $t$ with $z + i\omega/2a$.
        
        Notice that $\cscalar{e}|-az^2|$ is differentiable on the entire complex domain, and, therefore, is \emph{holomorphic} on $\complex$. Consider the following closed rectangle contour on the complex plane (see \cref{fig:demonstration-of-the-closed-contour}):
        \begin{equation*}
            (\alpha, -\omega/2a) \rightarrow (R, -\omega/2a) \rightarrow (R, 0) \rightarrow (\alpha, 0) \rightarrow (\alpha, -\omega/2a).
        \end{equation*}

        \begin{figure}
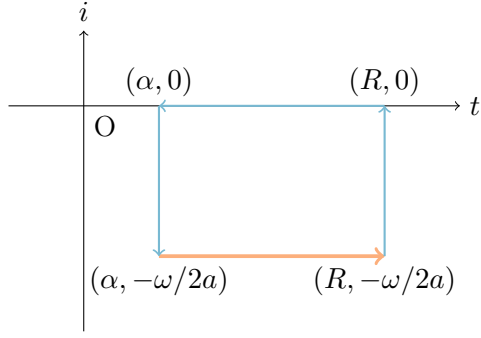

            \begin{center}
                \drawclosedcontour{1}{\omega/2a}{\alpha}
            \end{center}
            \caption{A demonstration of the closed contour.}
            \label{fig:demonstration-of-the-closed-contour}
        \end{figure}

        By \lemmaref{lma:cauchy-goursat-theorem} (Cauchy–Goursat Theorem), it holds that
        \begin{equation}
            \int_{\alpha - i\omega/2a}^{R - i\omega/2a} \cscalar{e}|-az^2| dz = - \int_{R - i\omega/2a}^R \cscalar{e}|-az^2| dz - \int_R^\alpha \cscalar{e}|-az^2| dz - \int_\alpha^{\alpha - i\omega/2a} \cscalar{e}|-az^2| dz. \label{eq:applying-cauchy-goursat-theorem}
        \end{equation}
        Observe that the first term on the RHS, $\int_{R - i\omega/2a}^R \cscalar{e}|-az^2| dz$, vanishes as $R\rightarrow+\infty$. Then, by taking the imaginary part of  Equation \eqref{eq:applying-cauchy-goursat-theorem} with $R\rightarrow +\infty$, we have that
        \begin{align}
            \image*[\int_{\alpha - i\omega/2a}^{R - i\omega/2a} \cscalar{e}|-az^2| dz] =& \image*[\int_{\alpha - i\omega/2a}^{\alpha} \cscalar{e}|-az^2| dz]
            \notag
            \\
            \ceq[i]& \image*[-i\int_{\omega/2\sqrt{a}}^0 \cscalar{e}|-a(\alpha - it/\sqrt{a})^2| dt]
            \notag
            \\
            =& \frac{\cscalar{e}|-a\alpha^2|}{\sqrt{a}}\image*[i\int_{0}^{\omega/2\sqrt{a}} \cscalar{e}|t^2 + i2\alpha\sqrt{a} t| dt]
            \notag
            \\
            =& \frac{\cscalar{e}|-a\alpha^2|}{\sqrt{a}} \int_{0}^{\omega/2\sqrt{a}} \cscalar{e}|t^2|\funcsbr{\cos}[2\alpha\sqrt{a} t] dt
            \label{eq:final-simplification-of-the-imaginary-part-of-the-partial-fourier-transform}
        \end{align}
        where Equation (i) holds because of the change of variables $z = \alpha - it/\sqrt{a}$, and the last equation holds by applying Euler's formula.

        Finally, plugging Equation \eqref{eq:final-simplification-of-the-imaginary-part-of-the-partial-fourier-transform} back in to Equation \eqref{eq:applying-euler-formula-and-change-variable-to-partial-fourier-transform} gives the claimed result.
    \end{proof}

    % Notice that \lemmaref{lma:centering-of-partial-sub-gaussian-ft} immediately gives the following corollary with the simple variable changing, $t = \omega u/2\sqrt{a}$.
    % \begin{corollary}
    % \label{cor:centering-of-partial-sub-gaussian-ft}
    %     For any $a > 0$, $\alpha\geq 0$, and $\omega\in\real$, it holds that
    %     \begin{equation*}
    %         \int_{\alpha}^{+\infty} \cscalar{e}|-at^2|\funcsbr{\sin}[\omega t] dt = \frac{\omega\cscalar{e}|-a\alpha^2 - \omega^2/4a|}{2a} \int_{0}^{1} \cscalar{e}|\omega^2 u^2/4a|\funcsbr{\cos}[\alpha\omega u] du.
    %     \end{equation*}
    % \end{corollary}

    % \begin{fact}[\emph{Equation 7.4.7} in \citet{abramowitz1948handbook}]
    % \label{fac:partial-fourier-transform-of-sub-gaussian-density-start-from-zero}
    %     For any $a > 0$ and $\omega\in\real$, it holds that
    %     \begin{equation*}
    %         \int_0^{+\infty} \cscalar{e}|-at^2|\funcsbr{\sin}[\omega t] dt = \frac{1}{\sqrt{a}}\cscalar{e}|-\omega^2/4a|\int_0^{\omega/2\sqrt{a}}\cscalar{e}|t^2| dt.
    %     \end{equation*}
    % \end{fact}

    Using \lemmaref{lma:centering-of-partial-sub-gaussian-ft}, we may derive a loose approximation as follows.
    \begin{corollary}
    \label{cor:partial-sub-gaussian-ft-lower-bound-intermediate}
        For any $a, \alpha > 0$ and $\omega\in\real$, if $\alpha\omega = 2k\pi$ for some $k\in\integer$, we have that
        \begin{equation*}
            \int_\alpha^{+\infty} \cscalar{e}|-at^2|\funcsbr{\sin}[\omega t] dt \geq \frac{\cscalar{e}|-2a\alpha^2 - \omega^2/8a|}{\sqrt{2a}}\int_0^{\omega/2\sqrt{2a}}\cscalar{e}|t^2|dt.
        \end{equation*}
    \end{corollary}
    \begin{proof}
        Notice that, by changing the variable $t = \cscalar{u} + \alpha$, we have that
        \begin{align*}
            \int_\alpha^{+\infty}\cscalar{e}|-at^2|\funcsbr{\sin}[\omega t] dt =& \int_0^{+\infty}\cscalar{e}|-a(\cscalar{u} + \alpha)^2|\funcsbr{\sin}[\omega (\cscalar{u} + \alpha)] d\cscalar{u}
            \\
            \ceq[i]& \int_0^{+\infty}\cscalar{e}|-a({\cscalar{u}|2|} + 2\alpha\cscalar{u} + \alpha^2)|\funcsbr{\sin}[\omega \cscalar{u}] d\cscalar{u}
            \\
            \cgeq[ii] &\cscalar{e}|-2a\alpha^2|\int_0^{+\infty}\cscalar{e}|-2a{\cscalar{u}|2|}|\funcsbr{\sin}[\omega \cscalar{u}] d\cscalar{u}
            \\
            =& \frac{\cscalar{e}|-2a\alpha^2 - \omega^2/8a|}{\sqrt{2a}}\int_0^{\omega/2\sqrt{2a}}\cscalar{e}| {\cscalar{u}|2|} | d\cscalar{u}
        \end{align*}
        where Equation (i) holds due to the assumption that $\alpha\omega = 2k\pi$ for some $k\in\integer$; Inequality (ii) is obtained by applying the elementary inequality $2\alpha\cscalar{u}\leq \alpha^2 + \cscalar{u}|2|$; and the last equation comes from an application of \lemmaref{lma:centering-of-partial-sub-gaussian-ft} on the integral with $\alpha = 0$ and $a\rightarrow 2a$.
    \end{proof}

    \begin{lemma}[Partial Gaussian FT Lower Bound]
    \label{lma:partial-fourier-transform-of-gaussian-density-start-from-alpha}
        For any $\alpha \geq 0$ and $\omega\geq 2$ such that $\alpha\omega = 2k\pi$ for some $k\in\integer$, it holds that $\int_{\alpha}^{+\infty}\funcsbr{\sin}[\omega \rscalar{z}]\funcsbr{\phi}[\rscalar{z}]d\rscalar{z} = \cscalar{e}|-\alpha^2|/4\omega$.
    \end{lemma}
    \begin{proof}
        Using \corollaryref{cor:partial-sub-gaussian-ft-lower-bound-intermediate} with $a = 1/2$, we have that
        \begin{align*}
            \int_{\alpha}^{+\infty}\funcsbr{\sin}[\omega \rscalar{z}]\funcsbr{\phi}[\rscalar{z}]d\rscalar{z} \geq& \frac{1}{\sqrt{2\pi}}\cscalar{e}|-\alpha^2 - \omega^2/4|\int_0^{\omega/2}\cscalar{e}|t^2|dt
            \\
            \cgeq[i]& \frac{2\cscalar{e}|-\alpha^2|}{3\omega\sqrt{2\pi}}
            \\
            >& \frac{\cscalar{e}|-\alpha^2|}{4\omega}
        \end{align*}
        where Inequality (i) is obtained by applying \lemmaref{lma:inverse-subgaussian-body-lower-bound} with $a = 1$ and $\beta = \omega/2$.
    \end{proof}
    
\section{Sub-Gaussian Tail Properties}
\label{sec:sub-gaussian-properties}
    % \begin{fact}[Gaussian Tail Bound]\label{fac:gaussian-tail-upper-bound}
    %     Let $\rscalar{z}\sim\gaussian[0][\sigma^2]$, we have $\prob{\rscalar{z}\geq t}\leq e^{-t^2/2\sigma^2}$.
    % \end{fact}
    
    \begin{lemma}[Sub-Gaussian Tail Lower Bound]
    \label{lma:subgaussian-tail-lower-bound}
        For any $\cscalar{a}, \alpha\geq 0$, it holds that $\int_\alpha^{+\infty} \cscalar{e}|-\cscalar{a} t^2|dt\geq\frac{\cscalar{e}|-\cscalar{a}\alpha^2|}{2\cscalar{a}\alpha + \sqrt{2\cscalar{a}}}$.
    \end{lemma}
    \begin{proof}
        Define $f:\real\rightarrow\real$ as
        \begin{equation*}
            f(u) = \int_u^{+\infty} \cscalar{e}|-\cscalar{a} t^2|dt - \frac{\cscalar{e}|-\cscalar{a} u^2|}{2\cscalar{a} u + \sqrt{2\cscalar{a}}}.
        \end{equation*}
        By the symmetry of $\gaussian[0][1/2\cscalar{a}]$, we have that
        \begin{align*}
            \sqrt{\frac{\cscalar{a}}{\pi}}\int_0^{+\infty} \cscalar{e}|-\cscalar{a} t^2|dt =& \prob<\rscalar{z}\sim\gaussian[0][1/2\cscalar{a}]>{\rscalar{z} \geq 0} 
            \\
            =& \frac{1}{2}
        \end{align*}
        which implies $f(0) = (\sqrt{\pi/2} - 1)/\sqrt{2\cscalar{a}} > 0$. Meanwhile, observe that, for $u\in[0, +\infty)$, it holds that
        \begin{align}
            \nabla_u f(u) =& -e^{-\cscalar{a} u^2}-\sbr{-\frac{2\cscalar{a}}{(2\cscalar{a} u+\sqrt{2\cscalar{a}})^2}e^{-\cscalar{a} u^2} -\frac{2\cscalar{a} u}{2\cscalar{a} u+\sqrt{2\cscalar{a}}}e^{-\cscalar{a} u^2}}
            \\
            =& -\frac{\sqrt{2}\cscalar{\cscalar{a}}|3/2|u}{(2\cscalar{a} u+\sqrt{2\cscalar{a}})^2}e^{-\cscalar{a} u^2/2}\\
            \leq& 0.
        \end{align}
        Eventually, because $\lim_{u\rightarrow +\infty} f(u) = 0$, $f(u)$ is always positive on $u\in[0, +\infty)$, which gives the claimed result.
    \end{proof}
    
    \begin{lemma}[Inverse Sub-Gaussian Tail Lower Bound]
    \label{lma:inverse-subgaussian-body-lower-bound}
        For any $\cscalar{a}, \beta > 0$ such that $\beta\geq \sqrt{1/\cscalar{a}}$, it holds that $\int_0^\beta \cscalar{e}|\cscalar{a}t^2|dt \geq \cscalar{e}|\cscalar{a}\beta^2|/3\cscalar{a}\beta$.
    \end{lemma}
    \begin{proof}
        Define $\funcsbr{g}[u] = \int_0^{u} \cscalar{e}|\cscalar{a}t^2|dt - \cscalar{e}|\cscalar{a}u^2|/3\cscalar{a}u$. Notice that, for any $t\geq 0$, we must have $\cscalar{e}|\cscalar{a}t^2|\geq \cscalar{a}t^2 + 1$. Thus, it holds that
        \begin{align*}
            \int_0^{\sqrt{1/\cscalar{a}}} \cscalar{e}|\cscalar{a}t^2|dt \geq& \int_0^{\sqrt{1/\cscalar{a}}} \cscalar{a}t^2 + 1 dt
            \\
            =& \frac{4}{3\sqrt{\cscalar{a}}}
            \\
            >& \frac{\cscalar{e}}{3\sqrt{\cscalar{a}}}
        \end{align*}
        which implies $\funcsbr{g}[\sqrt{1/a}]\geq 0$. Further, we have that
        \begin{align*}
            \funcsbr{g'}[u] =& \cscalar{e}|u^2| - \sbr*{\frac{2\cscalar{e}|u^2|}{3} - \frac{\cscalar{e}|u^2|}{3\cscalar{a}u^2}}
            \\
            =& \cscalar{e}|u^2|\sbr*{\frac{1}{3} + \frac{1}{3\cscalar{a}u^2}}
            \\
            \geq& 0
        \end{align*}
        which implies the claimed result.
    \end{proof}